\documentclass[11pt]{article} 
\usepackage[margin=1in]{geometry}

\usepackage{amssymb, amsmath, amsthm}
\usepackage{amsfonts,mathtools,mathrsfs,mathtools,color,bm}

\usepackage[bookmarksnumbered=true,
            bookmarksopen=true,
            colorlinks=true,
            citecolor=red,
            linkcolor=blue,
            anchorcolor=red,
            urlcolor=blue]{hyperref}

\usepackage{booktabs}       
\usepackage{amsfonts}       
\usepackage{nicefrac}       
\usepackage{microtype}      
\usepackage{xcolor}         
\usepackage{amsthm}
\usepackage{pdfpages}

\newtheorem*{rep@theorem}{\rep@title}
\newcommand{\newreptheorem}[2]{%
\newenvironment{rep#1}[1]{%
 \def\rep@title{#2 \ref{##1}}%
 \begin{rep@theorem}}%
 {\end{rep@theorem}}}

\newreptheorem{theorem}{Theorem}
\newreptheorem{lemma}{Lemma}

\newtheorem*{theorem*}{Theorem}
\newtheorem{fact}{Fact}
\newtheorem{lemma}{Lemma}
\newtheorem{definition}{Definition}
\newtheorem{corollary}{Corollary}
\newtheorem{example}{Example}

\newtheorem{assumption}{Assumption}

\newtheorem{proposition}{Proposition}
\newtheorem{theorem}{Theorem}

\usepackage{multirow}
\usepackage{mathtools}
\usepackage{amsthm, amsmath, amssymb}

\newcommand{\R}{\mathbb{R}} 
\newcommand{\N}{\mathcal{N}}
\newcommand{\E}{\mathbb{E}}

\renewcommand{\part}[2]{\frac{\partial #1}{\partial #2}}

\newcommand{\mgf}{\mathsf{mgf}}

\newcommand{\step}{h}
\newcommand{\error}{\varepsilon}

\newcommand{\Cov}{\mathrm{Cov}}

\def\Cov{\mathrm{Cov}}

\usepackage{prettyref}
\newrefformat{eq}{(\ref{#1})}
\newrefformat{thm}{Theorem~\ref{#1}}
\newrefformat{sec}{Section~\ref{#1}}
\newrefformat{algo}{Algorithm~\ref{#1}}
\newrefformat{fig}{Fig.~\ref{#1}}
\newrefformat{tab}{Table~\ref{#1}}
\newrefformat{rmk}{Remark~\ref{#1}}
\newrefformat{clm}{Claim~\ref{#1}}
\newrefformat{def}{Definition~\ref{#1}}
\newrefformat{cor}{Corollary~\ref{#1}}
\newrefformat{lmm}{Lemma~\ref{#1}}
\newrefformat{lem}{Lemma~\ref{#1}}
\newrefformat{prop}{Proposition~\ref{#1}}
\newrefformat{app}{Appendix~\ref{#1}}
\newrefformat{ex}{Example~\ref{#1}}

\begin{document}

\title{Convergence of the Inexact Langevin Algorithm in KL Divergence with Application to Score-based Generative Models}
\author{Kaylee Yingxi Yang\thanks{Department of Statistics and Data Science, Yale
University}, ~ 
Andre Wibisono\thanks{Department of Computer Science, Yale
University \\
\indent\indent \texttt{yingxi.yang@yale.edu}, ~ \texttt{andre.wibisono@yale.edu}}}

\date{}
\maketitle

\begin{abstract}
  Motivated by the increasingly popular Score-based Generative Modeling (SGM),
  we study the Inexact Langevin Dynamics (ILD) and Inexact Langevin Algorithm (ILA) where a score function estimate is used in place of the exact score. We establish {\em stable} biased convergence guarantees in terms of the Kullback-Leibler (KL) divergence. To achieve these guarantees, we impose two key assumptions:\ 1) the target distribution satisfies the log-Sobolev inequality, and 2) the error of score estimator exhibits a sub-Gaussian tail, referred to as Moment Generating Function (MGF) error assumption. Under the stronger $L^\infty$ score error assumption, we obtain a stable convergence bound in R\'enyi divergence.
  We also generalize the proof technique to SGM, and derive a stable convergence bound in KL divergence.
  In addition, we explore the question of how to obtain a provably accurate score estimator. We demonstrate that a simple estimator based on kernel density estimation fulfills the MGF error assumption for sub-Gaussian target distributions, at the population level.
\end{abstract}

\section{Introduction}

Score-based Generative Modeling (SGM) is a family of sampling methods which have achieved the state-of-the-art performance in many applications, including image, video and text generation~\cite{SE19, HJA20, CYA+20, SE20, SSDK21, CZZ21, AJH+21, LWYL2022, GLF+22, ZCP+22, YSM22, YZS+22}.
Motivated by the demonstrated empirical successes, the theoretical understanding of the SGM methods has been actively developed.
A crucial component of SGM is a good estimator of the {\em score function} (i.e.\ gradient of log-density) of the target distribution along a diffusion process.
Recent results~\cite{DBTHD21, debortoli2022, LLT22a, LLT22b, CLL2022, CCL+22} have established performance guarantees of SGM algorithms under some assumptions on the data distribution and error bounds on the score estimator.
Notably,~\cite{CCL+22, LLT22b, CLL2022} showed SGM enjoys a strong guarantee: the iteration complexity for SGM under general data distribution assumptions (such as smoothness and bounded second moment) matches the iteration complexity of the Langevin algorithm under isoperimetry.
The reverse process of SGM is built using the score functions of the distributions along the forward process, and
the guarantees of SGM hold assuming the score estimators have small error.
Motivated by these results, in this paper we study the problem of sampling with an inexact score function.

\subsection{Related work} 

In the case where exact evaluation of score function is computationally costly or even not available, many previous work including~\cite{HZ2017, DK2017, MMS2020} have studied Langevin algorithm using approximated score (e.g.\ via stochastic gradient). 
When the score estimator is random and has bounded bias and variance, the Wasserstein distance converges to a biased limit under strong log-concavity and smoothness assumptions, see~\cite[Theorem 3.4]{HZ2017}, \cite[Theorem 4]{DK2017} and \cite[Theorem 1.4]{MMS2020}. 
However, their assumptions on the error of score approximation
require a bounded $L^2$ error with respect to all distributions along the Langevin dynamics. This is satisfied e.g.\ when we have an $L^\infty$-accurate score estimator; otherwise, this is not an easily verifiable condition. 

There has been a surge of recent work in the theoretical analysis of SGM algorithms.~\cite{DBTHD21} studied the convergence in Total Variation (TV) under $L^\infty$ error assumption on the score estimator. Although $L^\infty$ is sufficient to ensure convergence, it may be overly stringent to satisfy in practice
since it requires a uniformly finite error at every point. \cite{BMR20} provided the first convergence result under $L^2$ error assumption; the result is in Wasserstein distance of order 2 but the error bound suffers from curse of dimensionality. \cite{LLT22a, debortoli2022} also studied convergence under $L^2$ error assumption. Their results are in TV and Wasserstein distance of order 1 respectively.
All the aforementioned work on SGM assumed either strong log-concavity or isoperimetry such as log-Sobolev inequality (LSI). More recently,~\cite{CCL+22, LLT22b, CLL2022} generalized the convergence analysis in TV and KL divergence to general data distributions with minimal assumptions such as smoothness and bounded second moment, without requiring isoperimetry, under $L^2$ score error assumption.
The convergence bounds in the works above typically diverge as $T \to \infty$, so the iteration complexity guarantees for SGM are derived by running the algorithm for a moderate amount of time, which cannot be too large.

\subsection{Contributions} 

In this paper, we study the Inexact Langevin Dynamics (ILD) and Inexact Langevin Algorithm (ILA), which are the classical Langevin dynamics and the unadjusted Langevin algorithm when we only have an inexact score estimator. This can be viewed as a special case of SGM, in which we don't have a forward process and only consider a single Langevin dynamics toward the target. We focus on the LSI target, in which case we have rapid convergence guarantees for KL divergence and R\'enyi divergence for Langevin dynamics and algorithms with exact score function (see Appendix~\ref{sec:review} for a review of the convergence results). We also derive a stable convergence bound for SGM in KL divergence, which does not diverge as running time increases. 
Our contributions can be summarized as follows:
\begin{enumerate}
    \item 
We establish biased convergence guarantees in KL divergence for both ILD and ILA. These convergence results are in line with state-of-the-art results for exact Langevin dynamics and ULA. A summary of our results can be viewed in Table \ref{tab:summary}. These results are obtained under the assumptions that the target distribution is LSI and the error in score estimation exhibits a sub-Gaussian tail. This is referred to as the Moment Generating Function (MGF) error assumption. 
Notably, our convergence bounds are {\em stable}, which means the upper bound remains bounded for all time.
Stable bounds provide better guarantees on the algorithm, e.g.\ an estimate on the asymptotic bias of the algorithm.

\renewcommand{\arraystretch}{1.2}
\begin{table*}
\caption{Comparison of Convergence Results for Exact and Inexact Langevin Dynamics and Algorithms} \label{tab:summary}
\begin{center}
\begin{tabular}{ |c|c|c|c|c| } 
\hline
Time & Score & Convergence under LSI & Reference \\
\hline
\multirow{2}{5em}{continuous} & exact & $H_\nu(\tilde \rho_t) \le e^{-2\alpha t} H_\nu(\rho_0)$ \hspace{30pt} &  \cite{OV2000} \\
& inexact & $H_{\nu}(\rho_t) \lesssim e^{-\frac{\alpha t}{2}} H_{\nu}(\rho_0) + \error_{\mgf}^2$ & Theorem \ref{thm:conv-dynamic} \\ 
\hline
\multirow{2}{5em}{discrete} & exact & $H_{\nu}(\tilde \rho_k) \lesssim e^{-\alpha \step k} H_{\nu}(\rho_0) + \step$ \hspace{30pt} & \cite{VW19} \\
& inexact & $ H_{\nu}(\rho_{k}) \lesssim e^{-\frac{\alpha hk}{4}} H_{\nu}(\rho_0) + h + \error_{\mgf}^2$ & Theorem \ref{thm:kl-bounded-MGF} \\ 
\hline
\end{tabular}
\end{center}
\end{table*}
\item Under the stronger $L^\infty$ error assumption on the score estimation, we prove convergence guarantees for ILA in R\'enyi divergence, which is stronger than KL divergence; see Theorem \ref{thm:renyi-max-error-bd}. The bound is also stable, which is important in applications such as differential privacy, in which R\'enyi divergence represents an important quantity (e.g.\ amount of information leaked) that we want to control and ensure remains small; e.g.\ see~\cite{GT2020, AT22}. 
We also present a convergence bound for R\'enyi divergence in the setting when the estimator is the score of another distribution which satisfies isoperimetry; see Appendix~\ref{sec:renyi-est-scorefxc}.

\item We generalize our proof to SGM and derive a convergence guarantee in KL divergence under LSI target and the MGF score error assumption; see Theorem~\ref{thm:conv-ddpm}. Contrary to previous results, our convergence result is again stable, indicating that the error is controlled and will not grow beyond a certain limit no matter how long the algorithm runs.  Unlike unstable bounds, stable bounds are more robust to the choice of running time $T$. In addition, we can read off the asymptotic bias of the algorithm from a stable bound. Theorem~\ref{thm:kl-bounded-MGF} implies that for ILA, as $k \to \infty$ and $h \to 0$, KL divergence is on the order of $\error_\mgf^2/\alpha$ and Theorem~\ref{thm:conv-ddpm} implies SGM has an asymptotic bias in KL divergence on the order of $\error_\mgf^2\log \frac{1}{\alpha}$.

\item We explore the question of how to get a provably accurate score estimator.
In Section~\ref{sec:kde-estimator} we demonstrate that when the target distribution is sub-Gaussian, a simple score estimator using Kernel Density Estimation (KDE) satisfies the MGF error assumption at the population level; see Lemma~\ref{lem-subgau-kde}. Consequently, it also satisfies the weaker $L^2$ assumption. As a result, we obtain an iteration complexity guarantee for ILA with a KDE-based score estimator; see Corollary~\ref{thm:cmplx}.

\end{enumerate}

\section{Problem setting}

Suppose we want to sample from a probability distribution $\nu$ on $\R^d$. We assume $\nu$ has full support on $\R^d$,
and it has a density function $\nu(x) \propto e^{-f(x)}$ with respect to the Lebesgue measure.
We assume $f \colon \R^d \to \R$ is differentiable.
The {\em score function} of $\nu$ is the vector-valued function $s_\nu \colon \R^d \to \R^d$ given by
\begin{align}\label{Eq:Score}
    s_\nu(x) = \nabla \log \nu(x) = -\nabla f(x).
\end{align}
In the classical setting of Langevin dynamics and algorithm, we assume we have access to $s_\nu$.
But in many practical settings, including SGM, we may only have an estimator of $s_\nu$.

\subsection{Sampling with exact score function via Langevin}

Suppose we have access to the score function $s_\nu = \nabla \log \nu = -\nabla f$.
Then one natural way to sample from $\nu$ is to run the Langevin dynamics in continuous time, which is the following stochastic process in $\R^d$:
\begin{align}\label{Eq:LD}
    dX_t = s_\nu(X_t) \, dt + \sqrt{2} \, dW_t
\end{align}
where $W_t$ is the standard Brownian motion in $\R^d$. There have been extensive studies on the convergence properties of the Langevin dynamics under various assumptions such as strong log-concavity or weaker isoperimetric inequality such as log-Sobolev inequality (LSI), which allows for some non-log-concavity~\cite{JK1998, OV2000}.
We recall that under LSI or Poincar\'e inequality, the Langevin dynamics converges to the target distribution $\nu$ exponentially fast (e.g.\ in KL divergence or chi-squared divergence). Under LSI, the Langevin dynamics also enjoys an exponentially fast convergence of the R\'enyi divergence~\cite{VW19}.
Conversely, if the target does not satisfy isoperimetry, then the Langevin dynamics may have slow convergence rate (e.g.\ when the target is multimodal).

In discrete time, a simple discretization of the Langevin dynamics is the Unadjusted Langevin Algorithm (ULA):
\begin{align}\label{Eq:ULA}
    x_{k+1} = x_k + \step \,s_\nu(x_k) + \sqrt{2\step} \, z_k
\end{align}
where $\step > 0$ is step size and $z_k \sim \N(0,I_d)$ is an independent standard Gaussian in $\R^d$.
We have convergence to a biased limit at a rate which matches the continuous-time convergence in KL and R\'enyi divergence under LSI (see Appendix~\ref{sec:review} for a review of the convergence results for the exact Langevin dynamics and ULA under LSI).

\subsection{Sampling with inexact score function via Langevin}

Suppose we only have an estimate $s \colon \R^d \to \R^d$ of the score function $s_\nu$ of $\nu$. As an analogy to Langevin dynamics, in continuous time we can run the following {\em Inexact Langevin Dynamics (ILD)}:
\begin{equation}
\label{Eq:ILD}
dX_t = s(X_t) \, dt + \sqrt{2} \, dW_t.
\end{equation}
If $s$ is a good estimator of $s_\nu$, then we might hope that the evolution of ILD~\eqref{Eq:ILD} approximately converges to $\nu$.
In discrete time, we can run the {\em Inexact Langevin Algorithm (ILA)}:
\begin{align}\label{Eq:ILA}
    x_{k+1} = x_k + \step \, s(x_k) + \sqrt{2\step} z_k
\end{align}
where $\step > 0$ is step size and $z_k \sim \N(0,I_d)$ is an independent standard Gaussian in $\R^d$. Under some error assumption between $s$ and $s_\nu$, we derive a biased convergence rate of the ILA~\eqref{Eq:ILA} to $\nu$.

\subsection{Sampling with score-based generative models}

In this section, we provide a brief review of a specific type of SGMs called Denoising Diffusion Probabilistic Modeling (DDPM) where the forward process is the Ornstein-Uhlenbeck (OU) process, a.k.a.\ variance preserving SDE; we refer to~\cite{HJA20} for more detail. 
Recall that DDPM proceeds via forward and backward processes as follows.
\paragraph{Forward Process}
For the forward process, we start from $X_0 \sim \nu_0 = \nu$ which is the data distribution, and follow the OU process targeting $\gamma = \N(0, I_d)$:
\begin{align}
\label{eq:ou}
dX_t = - X_t \, dt + \sqrt{2} dW_t.
\end{align}
Let $\nu_t \triangleq \text{Law}(X_t)$ be the measures along the OU flow above, and $s_t = \nabla \log \nu_t$ be the score function of $\nu_t$. The forward process can be interpreted as transforming samples from the data distribution $\nu$ into pure noise. Since the target measure $\gamma$ is LSI, we know $\nu_t \to \gamma$ exponentially fast.
\paragraph{Backward Process}
Suppose we run the forward process until time $T > 0$, ending at $\nu_T$.
If we reverse the forward SDE~\eqref{eq:ou} in time, then we convert the noise back into samples. This reversal allows us to generate new samples from $\nu$; and this can be achieved by the following SDE: 
\begin{align}\label{eq:backwardOU}
  d\tilde{Y}_{t} = (\tilde{Y}_{t} + 2 s_{T-t}(\tilde{Y}_{t})) dt + \sqrt{2} dW_t.
\end{align}
Let $\mu_t \triangleq \text{Law}(\tilde{Y}_t)$. 
If $\tilde{Y}_{0} \sim \mu_0 = \nu_T$, then by construction of~\eqref{eq:backwardOU} $\tilde{Y}_{t} \sim \mu_t = \nu_{T-t}$ for $0 \le t \le T$, so at time $T$, $\tilde{Y}_{T} \sim \mu_T = \nu$ is an exact sample from the target distribution~\cite{Anderson1982}.

However, in practice we do not know $\nu_T$ and $(s_t)_{0 \le t \le T}$, so we start the backward process at $\gamma$ the target distribution of forward process, and simulate the backward process in discrete time with a score estimator.

\paragraph{Algorithm} Based on the above, we consider the following algorithm, which aligns with other theoretical analysis works such as~\cite{CLL2022, CCL+22}. 
Let $\step > 0$ be the step size, and $K = \frac{T}{\step}$ so $T = K \step$ (assume $K \in \mathbb{N}$).
We construct a continuous-time process $(Y_t)_{0 \le t \le T}$ that starts from $Y_0 \sim \rho_0 = \gamma = \N(0, I_d)$.
In each step, from time $\step k$ to $\step(k+1)$, our process follows the SDE:
\begin{align}\label{eq:sgm-sde}
dY_{\step k + t} = (Y_{\step k + t} + 2 \hat s_{\step (K-k)}(Y_{\step k})) dt + \sqrt{2} dW_t
\end{align}
where $t \in [0, h]$, and $\hat s_{\step k}(y)$ is an approximation to $s_{\step k}(y)$, the score function of $\nu_{\step k}$. Then we update $y_{k+1}$ as the solution of the SDE~\eqref{eq:sgm-sde} at time $t=\step$ starting from $Y_{hk} = y_k$, i.e.
\begin{align}\label{eq:sgm}
y_{k+1} = e^{\step} y_k + 2(e^{ \step}-1) \hat s_{\step (K-k)}(y_k) + \sqrt{e^{2 \step}-1}\, z_k \qquad \textrm{\bf (DDPM)}
\end{align}
where $z_k \sim \N(0,I_d)$ is independent of everything so far. 
We refer to the algorithm~\eqref{eq:sgm} above as DDPM.

\subsection{Notations and definitions} 

In this section, we review notations and definitions of KL divergence, relative Fisher information and R\'enyi divergence. 
Let $\rho, \nu$ be two probability distributions in $\R^d$ denoted by their probability density functions w.r.t. Lebesgue measure on $\R^d$. Assume $\rho$ and $\nu$ have full support on $\R^d$, and they have differentiable log density functions.
\begin{definition}[KL divergence]
The Kullback-Leibler (KL) divergence of $\rho$ w.r.t. $\nu$ is 
\[H_{\nu}(\rho) = \int_{\R^d} \rho \log \frac{\rho}{\nu} dx.\]
\end{definition}

\begin{definition}[Relative Fisher information]
The relative Fisher information of $\rho$ w.r.t. $\nu$ is 
\[J_{\nu}(\rho) = \int_{\R^d} \rho \big\| \nabla \log \frac{\rho}{\nu}\big\|^2 dx.\]
\end{definition}

\begin{definition}[R\'enyi divergence]
For $q\geq 0$, $q\neq 1$, the R\'enyi divergence of order $q$ of $\rho$ w.r.t. $\nu$ is 
\[R_{q,\nu}(\rho) = \frac{1}{q-1} \log F_{q, \nu} (\rho)\]
where $F_{q, \nu} (\rho) = \E_{\nu}\left[ \left( \frac{\rho}{\nu}\right)^q\right]$.
\end{definition}
Recall when $q \to 1$, R\'enyi divergence recovers the KL divergence. Furthermore, $q \mapsto R_{q, \nu}(\rho)$ is increasing, so R\'enyi divergence bounds are stronger than KL divergence bounds. 

\subsection{Assumptions}
In this section, we introduce the assumptions that are necessary for our results.
\begin{assumption}[LSI]
\label{assump:lsi}
The target probability distribution $\nu$ is supported on $\R^d$ and satisfies LSI with constant $\alpha > 0$, which means for any probability distribution $\rho$ on $\R^d$:
\[ H_{\nu}(\rho) \leq \frac{1}{2\alpha} J_{\nu}(\rho).\]
\end{assumption}

We consider the following ways to measure the quality of the score estimator $s$ of $s_\nu = \nabla \log \nu$. For convergence in KL divergence, we require a MGF score error; for convergence in R\'enyi divergence, we assume the stronger $L^\infty$ assumption.

\begin{assumption}[MGF error assumption]
\label{bdd-mgf-assump} 
The error of $s$ has a finite moment generating function of some order $r > 0$ under $\nu$:
\[\error_{\mgf}^2 \equiv \error_{\mgf}^2(r,s,\nu) = \frac{1}{r}\log \E_\nu[\exp( r \|s(x) - s_\nu(x)\|^2)] < \infty.\]
\end{assumption}
\begin{assumption}[$L^\infty$ error assumption]
\label{assump:l-inf}
The error of $s(x)$ is bounded at every $x$, i.e.
\[\error_{\infty} = \sup_{x \in \R^d} \| s_{\nu}(x) - s(x) \| < \infty.\]
\end{assumption}

To establish convergence bound for discretized algorithms, we also assume the target measure has a Lipschitz score function, and the score estimator is also Lipschitz.

\begin{assumption}[$L$-smoothness]
\label{assump:lsmooth}
$f = -\log \nu$ is $L$-smooth for some $0 \le L < \infty$, which means $\nabla f \colon \R^d \to \R^d$ is $L$-Lipschitz: $\|\nabla f(x)-\nabla f(y)\| \le L\|x-y\|$ for all $x,y \in \R^d$.
\end{assumption}
\begin{assumption}[Lipschitz score estimator]
\label{assump:smoothscore}
The score estimator $s$ is $L_s$-Lipschitz for some $0 \le L_s < \infty$: $\|s(x)-s(y)\| \le L_s\|x-y\|$ for all $x,y \in \R^d$.
\end{assumption}

\section{Main results}
\label{Sec:Convergence}

\subsection{Convergence of ILD and ILA}
We first consider continuous time and compare ILD~\eqref{Eq:ILD} with the exact Langevin dynamics~\eqref{Eq:LD}. Recall that if the target distribution $\nu$ satisfies $\alpha$-LSI, then along the Langevin dynamics~\eqref{Eq:LD}, KL divergence is decreasing exponentially fast; see Appendix \ref{sec:review} for a brief review. When $s$ is an approximation of score function $s_{\nu}$ and it has a bounded MGF error, we show a similar convergence rate with an additional bias term induced by score estimation error.
\begin{theorem}[Convergence of KL divergence for ILD]
\label{thm:conv-dynamic}
Assume $\nu$ is $\alpha$-LSI and score estimator $s$ has a bounded MGF error with $r = \frac{1}{\alpha}$ (Assumptions~\ref{assump:lsi} and~\ref{bdd-mgf-assump}).
Then for $X_t \sim \rho_t$ along the ILD~\eqref{Eq:ILD} with score estimator $s$, we have
\[H_{\nu}(\rho_t) \le e^{-\frac{1}{2}\alpha t} H_{\nu}(\rho_0) + \frac{2}{\alpha} \,(1 - e^{-\frac{1}{2}\alpha t}) \, \error_{\mgf}^2.\]
\end{theorem}
The proof of Theorem \ref{thm:conv-dynamic} is in Appendix~\ref{sec:pf-conv-dynamic}. We note that the bound above is {stable, which means the right-hand side does not diverge as $t \to \infty$. 
This implies an estimate of the asymptotic bias of the ILD: $H_\nu(\nu_s) \le 2 \error_{\mgf}^2 / \alpha$, where $\nu_s$ is the limiting distribution of ILD with score estimator $s$. In discrete time, we also derive a convergence bound for ILA as follows.

\begin{theorem}[Convergence of KL divergence for ILA]
\label{thm:kl-bounded-MGF}
Assume $\nu$ is $\alpha$-LSI and $f=-\log \nu$ is $L$-smooth, and score estimator $s$ is $L_s$-Lipschitz and has bounded MGF error with $r = \frac{9}{\alpha}$ (Assumptions~\ref{assump:lsi},~\ref{bdd-mgf-assump},~\ref{assump:lsmooth} and~\ref{assump:smoothscore}). If $\,0 < h < \min(\frac{\alpha}{12 L_s L}, \frac{1}{2\alpha})$, then after $k$ iterations of ILA~\eqref{Eq:ILA},
\[ H_{\nu}(\rho_{k}) \le e^{-\frac{1}{4}\alpha hk} H_{\nu}{(\rho_0)} + C_1 d \,h + C_2 \,\error_{\mgf}^2\]
where $C_1=128 L_s (L_s + L)/\alpha$ and $C_2=8/(3 \alpha)$.
\end{theorem}
The proof is an extension of the interpolation technique of~\cite{VW19}. We sketch the major steps here and provide the full proof in Appendix~\ref{sec:pf-kl-bdd-MGF}.
\begin{proof}[Proof sketch of Theorem~\ref{thm:kl-bounded-MGF}]
We first note that one step of the ILA~\eqref{Eq:ILA} is the solution $x_{k+1}=X_h$ of the following interpolated SDE at time $t = \step$ starting from $X_0=x_k$:
\begin{equation}
\label{eq:interpolation-SDE}
    dX_t = s(X_0)dt + \sqrt{2}dW_t
\end{equation}
where $W_t$ is the standard Brownian motion in $\R^d$, and $t$ is from $0$ to $h$. We can bound the time derivative of KL divergence along~\eqref{eq:interpolation-SDE} as follows.
\begin{lemma}
\label{lemma-time-deriv-kl}
Suppose the assumptions in Theorem \ref{thm:kl-bounded-MGF} hold. Let $\rho_t \triangleq \text{Law}(X_t)$ where $X_t$ follows SDE~\eqref{eq:interpolation-SDE}, then
\[\frac{\partial}{\partial t}H_{\nu}{(\rho_t)} \leq -\frac{3}{4} J_{\nu}(\rho_t) + \E_{\rho_{0t}}\left[ \| s(x_0) - \nabla \log \nu(x_t) \|^2 \right].\]
\end{lemma}
Additionally, we would need the following bound on $\| s(x_t) - s(x_0) \|^2$.
\begin{lemma}
\label{score-lipschitz}
If the score estimator $s(x)$ is $L_s$-Lipschitz and $t \leq \frac{1}{3L_s}$, then for any fixed $z_0$,
\begin{align*}
\| s(x_t) - s(x_0) \|^2 \leq 18 L_s^2 t^2 \| s(x_t)  - \nabla \log \nu (x_t) \|^2  + 18 L_s^2 t^2 \| \nabla \log \nu(x_t) \|^2 + 6L_s^2 t \| z_0 \|^2
\end{align*}
where $x_t = x_0 + t s (x_0) + \sqrt{2t}z_0$.
\end{lemma}
Using Lemma~\ref{lemma-time-deriv-kl} and~\ref{score-lipschitz}, along with the Donsker-Varadham representation of the KL divergence for a change of measure, we obtain the following one-step contraction.
\begin{lemma}
\label{lemma-one-step-conv-mgf}
Suppose the assumptions in Theorem \ref{thm:kl-bounded-MGF} hold, then along each step of ILA~\eqref{Eq:ILA},
\begin{align*}
H_{\nu}(\rho_{k+1}) &\leq e^{-\frac{1}{4}\alpha h} H_{\nu}{(\rho_{k})} + 144 d L_s^2 L h^3 + 24 d L_s^2 h^2 +\frac{9}{2} \error_{\mgf}^2 \,h.    
\end{align*}
\end{lemma}
Applying the above contraction recursively for $k$ times yields the desired result.
\end{proof}

We note that the first two terms in the upper bound above match state-of-the-art result for ULA~\eqref{Eq:ULA} (with exact score function) under LSI: $H_{\nu}(\rho_k) \lesssim e^{-\alpha \step k} H_{\nu}(\rho_0) + \frac{d L^2}{\alpha}\step$; see Table~\ref{tab:summary}. In ILA case, there is an extra non-vanishing term induced by the error of score estimator. 
So in order to have a small asymptotic error, we need an accurate score estimator with a small MGF error. In practice, score matching is used, but its theoretical understanding lags behind. 
In section~\ref{sec:kde-estimator}, we show that if we use a score estimator based on Gaussian KDE (at the population level), then the score estimator satisfies the MGF error assumption.

\subsubsection{Convergence in R\'enyi divergence}

We also derive a stable convergence rate for ILA in R\'enyi divergence, which is stronger than KL divergence, under the stronger $L^\infty$ error assumption. The bound matches with that for ULA but has an extra non-vanishing term resulting from the error introduced by score estimation.
\begin{theorem}[Convergence of R\'enyi divergence for ILA under $L^\infty$ error]
\label{thm:renyi-max-error-bd}
Assume $\nu$ is $\alpha$-LSI and $f=-\log \nu$ is $L$-smooth, and score estimator $s$ is $L_s$-Lipschitz and has a bounded $L^\infty$ error (Assumptions~\ref{assump:lsi},~\ref{assump:l-inf},~\ref{assump:lsmooth} and~\ref{assump:smoothscore}). Let $q \geq 1$. If $\,0 < h < \min(\frac{\alpha}{12 L L_s q}, \frac{q}{4\alpha})$, then after $k$ iterations of ILA~\eqref{Eq:ILA},
\[ R_{q, \nu}(\rho_{k}) \leq e^{-\frac{\alpha hk}{q}}  R_{q, \nu}(\rho_{0}) + C_1 d h + ( C_2 h^2 + C_3)\, \error_{\infty}^2,\]
where $C_1=\frac{16 L_s q}{\alpha}(L + 2L_s q)$ and $C_2=\frac{96 L_s^2 q^2}{\alpha}$ and $C_3 = \frac{16 q^2}{3\alpha}$.
\end{theorem}
The proof adapts the interpolating method from KL divergence to R\'enyi divergence. The detailed proof for Theorem~\ref{thm:renyi-max-error-bd} can be found in Appendix~\ref{sec:pf-renyi-max-error-bd}.

However, we note that an $L^\infty$-accurate score estimator requires a uniformly finite error at every point, which is quite strong and may not hold in practice (see Example~\ref{ex:GaussianEx} below). 
If MGF error assumption is used instead of $L^\infty$, the current proof technique breaks at~\eqref{eq:lem8eq2} since the score error is measured with respect to some measure different from $\rho_k$ and we cannot perform a change of measure argument to relate it in terms of R\'enyi divergence of $\rho_k$ w.r.t.\ $\nu$. But we anticipate the bound holds under MGF error, analogous to the case of KL divergence, and leave it for future work. 
\begin{example}[Comparison of different score error assumptions in a simple Gaussian case]
\label{ex:GaussianEx}
Let $\nu = \N(0, {\alpha^{-1}}I_d)$, so $s_\nu(x) = -\alpha x$. Suppose we estimate $s_\nu$ by the score of $\N(0, {\hat{\alpha}}^{-1}I_d)$, so $\hat{s}(x) = -\hat{\alpha} x$.
Assume $\hat \alpha \neq \alpha$, which means we did not estimate the variance of the distribution correctly.
Then the $L^\infty$ error is unbounded: 
\[\error_{\infty} = \sup_{x \in \R^d} \|s_\nu(x) - \hat s(x)\| = |\hat{\alpha} - \alpha| \sup_{x \in \R^d} \| x \| = \infty.\]
On the other hand, the $L^2$ error is bounded: $\error_2^2 = (\hat{\alpha} - \alpha)^2 \E_{\nu} [\|X\|^2] = \frac{(\hat{\alpha} - \alpha)^2 d}{\alpha}.$ The $L^p$ error ($p$-th order moment of $\nu$) is also bounded for all $p < \infty$.
The MGF error is also bounded for $r < \frac{\alpha}{2 (\hat{\alpha} - \alpha)^2}:$
\begin{align*}
\error_\mgf^2 &\coloneqq \error_{\mgf}^2(r, \hat{s}, \nu) = \frac{1}{r} \log \E_\nu[e^{r(\hat{\alpha} - \alpha)^2 \|X\|^2}] = \frac{d}{2r} \log \left(\frac{\alpha}{\alpha - 2r(\hat{\alpha} - \alpha)^2}\right).  
\end{align*}

\end{example}

\paragraph{When score estimator is a score function}
Our results thus far are applicable to any estimator $s$ of the score function $s_\nu$ that satifies the assumptions of being Lipschitz and having bounded MGF (or $L^\infty$) error. It is common for the estimator to actually be the score function of another distribution $\hat{\nu}$ which approximates $\nu$, as in the example above. Consequently, when employing the score estimator $s = s_{\hat \nu}$, ILD for $\nu$ becomes equivalent to running the exact Langevin dynamics for $\hat \nu$. Therefore, we can characterize the performance of ILD by considering the performance of Langevin dynamics for $\hat \nu$.
For example, when $\hat \nu$ satisfies LSI, we know that R\'enyi divergence with respect to $\hat \nu$ converges exponentially fast for both Langevin dynamics and ULA targeting $\hat \nu$; see Appendix~\ref{sec:review}. Combining the exponential convergence rate with the generalized triangle inequality for R\'enyi divergence, we can obtain a biased convergence rate of R\'enyi divergence with respect to $\nu$ for both ILD and ILA. 
The formal statements and proofs of the bounds can be found in Appendix~\ref{sec:renyi-est-scorefxc}, where we also provide a detailed comparison with Theorem~\ref{thm:renyi-max-error-bd}.

\subsection{Application to score-based generative models}
\label{sec:ddpm}
In this section, we generalize our proof technique from ILA to DDPM. To derive a stable convergence bound for DDPM, an additional assumption is required:

\begin{assumption}[MGF error for SGM]
\label{assumption:mgf-sgm} 
For any $k \in [K]$, $\hat{s}_{k\step}$ has a bounded error against the continuous-time true score over the interpolation interval $[0,\step)$:
\[\error_{\mgf}^2 \coloneqq  \sup_{k \in [K]} \sup_{t \in [0,\step)}\frac{1}{r_{k\step - t}} \log \E_{\nu_{k \step - t}}[\exp(r_{k\step - t} \|\hat s_{k \step} - s_{k \step - t}\|^2)] < \infty.\]
\end{assumption}
We now present our convergence result as follows.
\begin{theorem}
\label{thm:conv-ddpm}
Assume the target distribution $\nu$ is $\alpha$-LSI $(\alpha < 1/2)$ and $L$-smooth, and for any $k \in [K]$ the score estimator $\hat{s}_{k \step}$ is $L_s$-Lipschitz and has a bounded MGF error with $r_{k\step - t}=\frac{65}{6\alpha_{k\step - t}}$ where $t \in [0, \step)$ and $\alpha_{k\step - t} = \frac{\alpha}{\alpha + (1-\alpha) e^{-2({k\step - t})}}$.  (Assumptions~\ref{assump:lsi},~\ref{assump:lsmooth},~\ref{assump:smoothscore} and~\ref{assumption:mgf-sgm}).
Let $\rho_k$ be the law of output of Algorithm~\eqref{eq:sgm} at the $k$-th step starting from $\rho_0 = \gamma = \N(0, I_d)$. If $0 < \step \le \frac{1}{96 L_s L}$, then
\[H_{\nu}(\rho_{K}) \lesssim \alpha^{-5/4} e^{-\frac{5Kh}{2}} H_{\nu}(\gamma)+ \left(  \error_{\mgf}^2 +  L_s^2 d\step\right)\log \frac{1}{\alpha}.\]
\end{theorem}
We interpret the bound above in the following manner: The first term arises from the initialization error of the algorithm. Recall that the true backward process should start from $\nu_T$ (the terminal distribution of forward process), but the algorithm starts from $\gamma = \N(0, I_d)$. 
The second term is from the score estimation error which is non-vanishing.
The third term is discretization error which scales with step size $h$. 
Compared to the results in \cite{CCL+22, LLT22b}, our convergence result is stable (the bound above does not diverge as $K \to \infty$). But we note that this stability comes at the cost of requiring stronger assumptions, including LSI and MGF error. Theorem~\ref{thm:conv-ddpm} directly implies the following complexity result.
\begin{corollary}
Suppose the assumptions in Theorem \ref{thm:conv-ddpm} hold. For any $\varepsilon > 0$, if the score estimator $\hat{s}$ along the forward process has MGF error $\error_{\mgf}^2 =O(\error/\log\frac{1}{\alpha})$, then running the DDPM with step size $\step = O\left( \dfrac{\error }{d L_s^2 \log \frac{1}{\alpha}} \right)$ for $K=\Omega\left(\dfrac{d L_s^2\log \frac{1}{\alpha}}{\error} \log \dfrac{ H_{\nu}(\gamma)}{\error \alpha^{5/4}} \right)$ will reach a distribution $\rho_K$ with $H_{\nu}(\rho_{ K}) \le \varepsilon$.
\end{corollary}
We sketch the major steps of the proof of Theorem~\ref{thm:conv-ddpm} here and provide the full proof in Appendix~\ref{sec:pf-conv-ddpm}.
\begin{proof}[Proof sketch of Theorem~\ref{thm:conv-ddpm}]
The proof is via extending the interpolation approach from the static Langevin case to DDPM. Specifically, we compare the evolution of KL divergence along one step of the DDPM with the evolution along the true backward process~\eqref{eq:backwardOU} in continuous time. For simplicity, suppose $k=0$, so we start the backward process at $\tilde{y}_{0} \sim \mu_0$ 
and the DDPM at $y_0 \sim \rho_0$.
Then we can write one step of the DDPM
\[y_{1} = e^{ \step} y_0 + 2(e^{ \step}-1)\hat s_{T_0}(y_0) + \sqrt{e^{2 \step}-1}\, z_0\]
where $T_k = T - kh$
as the output at time $t=h$ of the following SDE
\begin{align}\label{eq:interp-sde}
dY_t = (Y_t + 2 \hat{s}_{T_0}(Y_0)) dt + \sqrt{2} dW_t,\quad t \in [0, \step]
\end{align}
where $W_t$ is the standard Brownian motion in $\mathbb{R}^d$ starting at $W_0 = 0$. 

The key of our proof is the following lemma which bounds contraction of KL divergence along one iteration of DDPM.
\begin{lemma}\label{lemma:one-step-ddpm}
Suppose the assumptions in Theorem \ref{thm:conv-ddpm} hold. Let $\mu_k$ be the distribution at $t=\step k$ along the true backward process~\eqref{eq:backwardOU} starting from $\mu_0 = \nu_T$, so that $\mu_k = \nu_{T-hk}$ and $\mu_K = \nu_0 = \nu$. Let $\rho_k$ be the distribution of output of~\eqref{eq:sgm} at $k$-th step starting from $\rho_0 = N(0, I_d)$. Assume $0 < \step \le \frac{1}{96 L_s L}$.
Then for all $k = 0,1,\dots,K-1$,
\begin{align*}
H_{\mu_{k+1}}(\rho_{k+1}) &\le \left(\frac{\alpha e^{2T_{k+1}} + 1-\alpha}{\alpha e^{2T_k} + 1-\alpha}\right)^{1/4}  H_{\mu_{k}}(\rho_{k})  + \frac{65}{8} \error_{\mgf}^2 h+ \frac{9}{2} L_s (3 + 32 L_s)\, d\step^2.
\end{align*}
\end{lemma}
We provide the proof of Lemma \ref{lemma:one-step-ddpm} in Appendix \ref{sec:pf-lemma-one-step-ddpm}. Recursively applying Lemma \ref{lemma:one-step-ddpm} will give us the contraction of KL divergence starting from $H_{\mu_0}(\rho_0) = H_{\nu_T} (\gamma)$. Recall that $\nu_T$ is the distribution at time $T$ along the OU process with target distribution $\gamma$, and that the OU process converges exponentially fast. Note that here we are measuring KL divergence with respect to $\nu_T$ (instead of $\gamma$), but the convergence still holds, see the following lemma.
\begin{lemma}\label{lem:conv-ou}
Assume $\nu$ is $\alpha$-LSI ($\alpha > 0$).
Let $X_0 \sim \nu_0 = \nu$ and $X_t \sim \nu_t$ evolve along the OU process targeting $\gamma = \N(0, \alpha^{-1}I_d)$.
Then at any $T \ge 0$:
    \[H_{\nu_T} (\gamma) \le \frac{H_{\nu_0}(\gamma)}{\alpha e^{2T} + 1-\alpha}. \]
\end{lemma}

We provide the proof of Lemma \ref{lem:conv-ou} in Appendix \ref{sec:pf-lsi-ou}. Theorem \ref{thm:conv-ddpm} then follows from Lemma \ref{lemma:one-step-ddpm} and Lemma \ref{lem:conv-ou}; see complete proof in Appendix \ref{sec:pf-conv-ddpm2}. 
\end{proof}

\section{KDE-based score estimator}
\label{sec:kde-estimator}

A crucial ingredient in the SGM results~\cite{CCL+22, LLT22b}, as well as in our work, is a good score estimator that satisfies the required error assumption.
In practice, methods such as score matching~\cite{HD05, H07}~ are popular and have been successful in producing impressive empirical results, although the theoretical guarantees are still being developed.
There are classical methods such as Maximum Likelihood Estimator (MLE) that can also be used~\cite{KHR2022}.

Here we study a simple score estimator based on Kernel Density Estimation (KDE), in particular with Gaussian kernel. Given i.i.d.\ samples $X_1,\dots, X_n$ from the unknown measure $\rho$, we consider the following kernel density estimator with  bandwidth $\eta > 0$
\[\hat \rho_\eta = \frac{1}{n} \sum_{i=1}^n \N(X_i, \eta I_d).\] 
Then we estimate $s = \nabla \log \rho$ by the score of this KDE:
\begin{align}
\label{eq:kde-score-data}
\hat s_\eta(y) = \nabla \log \hat \rho_\eta (y)= \frac{\sum_{i=1}^n (X_i-y) e^{-\frac{\|y-X_i\|^2}{2\eta}}}{\eta \sum_{i=1}^n e^{-\frac{\|y-X_i\|^2}{2\eta}}}.
\end{align}

We analyze the performance of this estimator in the population level ($n\to \infty$).
This means we estimate the score $s = \nabla \log \rho$ by 
$\hat s_\eta = \nabla \log \rho_\eta$
where $\rho_\eta = \rho \ast \N(0,\eta I_d)$.
We show that when the data distribution $\rho$ is sub-Gaussian, the KDE score estimator with sufficiently small $\eta$ satisfies the MGF error assumption (and thus also satisfies the $L^2$ error assumption). 
Recall we say that a probability distribution $\rho$ is $\alpha$-sub-Gaussian for some $0 \le \alpha < \infty$ if $\E[\exp(v^\top (X - \E X))] \le \exp(\frac{\alpha^2 \|v\|^2}{2})$ for all $v \in \R^d$.
\begin{lemma}
\label{lem-subgau-kde}
For $\eta \ge 0$. Let $\rho_\eta = \rho \ast \N(0, \eta I_d)$ and
$s_\eta = \nabla \log \rho_\eta$.
Assume $\rho$ is $L$-smooth and $\sigma$-sub-Gaussian for some $0 < \sigma < \infty$.
For all $r > 0$ and for all $0 \le \eta \le \min\{\frac{d}{\|s(0)\|^2}, \frac{1}{2\sqrt{2} \sigma \sqrt{r} L^2}\}$, 
\[\error_{\mgf}^2 = \frac{1}{r}\log \E_\rho[e^{r\|s_\eta-s\|^2}] \lesssim \eta L^2 (d + \sigma^2\eta L^2)\]
where $\lesssim$ hides absolute constant.
\end{lemma}

The proof of Lemma \ref{lem-subgau-kde} is in Appendix \ref{sec:pf-subgau-kde}. 
Recall that $\alpha$-LSI ($\alpha > 0$) implies $\frac{1}{\alpha}$-sub-Gaussian \cite{Ledoux99}.
It follows that the KDE score estimator with a sufficiently small bandwidth has a small MGF error for LSI target distribution. 
Leveraging Theorem~\ref{thm:kl-bounded-MGF}, we obtain the following complexity result of ILA with a KDE-based score estimator.

\begin{corollary}
\label{thm:cmplx}
Assume the target distribution $\nu$ is $\alpha$-LSI ($\alpha > 0$) and $L$-smooth. For any $\varepsilon > 0$, suppose we estimate the score function $s_{\nu}$ by $\hat s = \nabla \log \nu *\N(0, \eta I_d)$ using Gaussian kernel (at the population level) with bandwidth $\,\eta= O\left( \dfrac{\error \alpha}{dL^2}\right)$. Then running ILA~\eqref{Eq:ILA} with score estimator $\hat s$ and step size $h=O\left(\dfrac{\error \alpha}{dL^2 }\right)$ for at least $k = O\left(\dfrac{dL^2}{\error \alpha^2}\log \dfrac{H_{\nu}(\rho_0)}{\error}\right)$ iterations reaches $H_{\nu}(\rho_{k}) \leq \error$.
\end{corollary}

An interesting question is to investigate the finite-sample error of score estimation. This question was explored in \cite{CHZ+2023}, where neural networks were employed to study the subject. The authors derived a finite-sample bound by carefully selecting a network architecture and tuning the parameters. The obtained results require the target distribution being sub-Gaussian and the data lying in a low dimensional linear subspace.
More recently,~\cite{wibisono2024optimal} studied this question for a more general data distribution, only requiring sub-Gaussianity and Lipschitz score, and derived a finite sample error bound for a KDE-based score estimator.


\section{Conclusion and discussion}
In this paper, we derived stable convergence guarantees of ILD and ILA in KL divergence for LSI target distribution under the assumption that the score estimator error has a bounded MGF error (Assumption \ref{bdd-mgf-assump}). 
The MGF error is weaker than the $L^\infty$ error assumption (which may be too strong to hold in practice) yet stronger than the $L^2$ error bound assumption which has been used in recent work for SGM, albeit resulting in unstable convergence bounds. Under the $L^\infty$ error assumption, we derived a stable convergence result in R\'enyi divergence for ILA. We also generalized the proof technique for ILA to the setting of SGM and obtained a convergence bound for SGM in KL divergence for LSI target under the MGF error assumption.
A feature of our result is that our convergence bounds are stable: the upper bound remains finite as $T \to \infty$, and thus gives an estimate on the asymptotic bias. We also demonstrated that a simple KDE-based score estimator satisfies the MGF error assumption at the population level for sub-Gaussian target.

This work has the following limitations. Firstly, our convergence results in KL divergence are established under the MGF score error assumption. However, it remains an open question whether it is feasible to prove an analogous stable convergence bound under the weaker assumption of $L^2$ score error. It would be valuable to investigate if such a result can be derived, or if there exists a counterexample demonstrating its impossibility.
Furthermore, we have not addressed the finite-sample MGF bound for score estimation, which presents an exciting avenue for further research. In another future direction, our proof technique can be extended to examine the convergence guarantees of other sampling algorithms that utilize inexact score functions, including the underdamped Langevin dynamics. By applying similar analysis techniques, we can explore the convergence properties and establish theoretical guarantees for these algorithms as well. Additionally, it would be interesting to study the stability of sampling algorithms when we change the target distribution or its score function.


\bibliographystyle{amsalpha}
\bibliography{ref}

\appendix
\section{Review on convergence results for Langevin dynamics and ULA}
\label{sec:review}
In this section, we review some convergence results for Langevin dynamics and ULA under LSI case, following~\cite{OV2000, VW19, CEL+22}. Convergence results are also available under weaker assumptions such as Poincar\'e inequality (PI), but PI is not in the scope of this paper so we skip results under PI.

We first recall that under LSI assumption, KL divergence decreases exponentially fast to 0 along the Langevin dynamics.
\begin{fact}[Convergence of KL divergence for Langevin dynamics]
Suppose $\nu$ satisfies LSI with constant $\alpha > 0$. Along the Langevin dynamics~\eqref{Eq:LD},
\[H_{\nu}(\rho_t) \le e^{-2\alpha t} H_{\nu}(\rho_0).\]
\end{fact}
R\'enyi divergence also converges exponentially fast along the Langevin dynamics. Note that when $q\to1$ (R\'enyi divergence which is defined via a limit recovers KL divergence), the following result recovers the exponential convergence of KL divergence.
\begin{fact}[Convergence of R\'enyi divergence for Langevin dynamics]
\label{fact:renyi-langevin-conti}
Suppose $\nu$ satisfies LSI with constant $\alpha > 0$. Let $q \ge 1$. Along the Langevin dynamics~\eqref{Eq:LD},
\[R_{q, \nu}(\rho_t) \le e^{-\frac{2\alpha t}{q}} R_{q, \nu}(\rho_0).\]
\end{fact}
However, in discrete time, ULA is an inexact discretization of the Langevin dynamics. When the target distribution $\nu$ is LSI and log-smooth, ULA converges exponentially fast to a biased limit.
\begin{fact}[Convergence of KL divergence for ULA]
Suppose $\nu$ satisfies LSI with constant $\alpha > 0$ and is $L$-smooth. For any $X_0 \sim \rho_0$ with $H_{\nu}(\rho_0) < \infty$ and step size $0 < h \le \frac{\alpha}{4L^2}$, then after $k$ iterations of ULA~\eqref{Eq:ULA},
\[ H_{\nu}(\rho_{k}) \le e^{-\alpha hk} H_{\nu}{(\rho_0)} + O\left(\frac{ dhL^2}{\alpha}\right).\]
\end{fact}
As $h \to 0$, ULA recovers the continuous-time Langevin dynamics. However for any fixed $h > 0$, as $k \to \infty$, KL divergence does not go to 0, it has an asymptotic bias scaling with step size $h$. 

Last we recall the convergence of R\'enyi divergence for ULA which was discovered more recently~\cite{CEL+22}.
\begin{fact}[Convergence of R\'enyi divergence for ULA]
\label{fact:renyi-ula}
Suppose $\nu$ satisfies LSI with constant $\alpha > 0$ and is $L$-smooth. Assume $q \ge 3$. For any $X_0 \sim \rho_0$ with $R_{2, \nu}(\rho_0) < \infty$ and step size $0 < h \le \frac{\alpha}{192q^2L^2}$, then after $k \ge K_0 \coloneqq \frac{2}{\alpha h} \log \frac{q-1}{2}$ iterations of ULA~\eqref{Eq:ULA},
\[ R_{q,\nu}(\rho_{k}) \le \exp\left(-\frac{\alpha h (k-K_0)}{4}\right) R_{2, \nu}(\rho_0) + \tilde{O}\left( \frac{dhqL^2}{\alpha}\right)\]
\end{fact}
For any fixed $h >0$, as $k \to \infty$, we obtain an asymptotic estimate of the bias: $\tilde{O}\left( \dfrac{dhqL^2}{\alpha}\right)$.

\section{Proof of Theorem~\ref{thm:conv-dynamic}}
\label{sec:pf-conv-dynamic}
\begin{proof}
The Fokker-Planck equation of the Langevin dynamics \eqref{Eq:ILD} is
\[\frac{\partial \rho_{t}}{\partial t} = \nabla \cdot (-\rho_t s) + \Delta \rho_t.\]
It follows that the time derivative of KL divergence can be written as
\begin{align*}
    \frac{\partial}{\partial t}H_{\nu}{(\rho_t)} &= \int_{\R^d} \frac{\partial \rho_t}{\partial t}\log \frac{\rho_t}{\nu} dx \\
    &= \int_{\R^d} \left(\nabla \cdot (-\rho_t s) + \Delta \rho_t \right) \log \frac{\rho_t}{\nu} dx \\
    &= \int_{\R^d} \left(-\nabla \cdot (\rho_t s ) + \nabla \cdot (\rho_t \nabla \log \frac{\rho_t}{\nu}) + \nabla \cdot (\rho_t \nabla \log \nu)\right) \log \frac{\rho_t}{\nu} dx \\
    &= \int_{\R^d} \left( \nabla \cdot \left(\rho_t (\nabla \log \frac{\rho_t}{\nu} - s + \nabla \log \nu)\right)\right)\log \frac{\rho_t}{\nu} dx  \\
    &= -\int_{\R^d}  \rho_t \langle \nabla \log \frac{\rho_t}{\nu} - s + \nabla \log \nu, \nabla \log \frac{\rho_t}{\nu} \rangle dx \qquad \text{\footnotesize by integration by parts}\\
    &= -\int_{\R^d}  \rho_t \| \nabla \log \frac{\rho_t}{\nu} \|^2 dx + \int_{\R^d} \rho_t \langle s - \nabla \log \nu, \nabla \log \frac{\rho_t}{\nu} \rangle dx \\
    &= -J_{\nu}(\rho_t) + \int_{\R^d} \rho_t \langle s - \nabla \log \nu, \nabla \log \frac{\rho_t}{\nu} \rangle dx \\
    &\le -J_{\nu}(\rho_t) + \E_{\rho_{t}}\left[ \| s - \nabla \log \nu \|^2 \right] + \frac{1}{4} \E_{\rho_{t}}\left[ \| \nabla \log \frac{\rho_{t}}{\nu} \|^2 \right] \qquad \text{\footnotesize by $\langle a, b \rangle \leq \| a \|^2 + \frac{1}{4} \| b \|^2$} \\
    &= -\frac{3}{4}J_{\nu}(\rho_t) + \E_{\rho_{t}}\left[ \| s - \nabla \log \nu \|^2 \right].
\end{align*}
Recall the following Donsker-Varadhan variational characterization of KL divergence: Let $P, Q$ be two measures on $\R^d$ and $f: \R^d \to \R$ be any function such that $\E_{Q}e^{f(x)} < \infty$, then
\[\E_{P}[f(x)] \leq \log \E_{Q}e^{f(x)} + H_Q(P).\]
It follows that we can bound the $L^2(\rho_t)$ error in terms of MGF error as follows:
\begin{align}
 r\E_{\rho_{t}}\left[ \| s - \nabla \log \nu \|^2 \right] &\le \log \E_\nu[\exp(  r\|s(x) - s_\nu(x)\|^2)] +H_{\nu}(\rho_t) \notag \\
 &= r \error_{\mgf}^2 + H_{\nu}(\rho_t) \label{eq:dv-kl}
\end{align}
for $r=\frac{1}{\alpha}$. Therefore,
\begin{align*}
    \frac{\partial}{\partial t}H_{\nu}{(\rho_t)} &\le -\frac{3}{4}J_{\nu}(\rho_t) + \error_{\mgf}^2 + \alpha H_{\nu}(\rho_t) \\
    &\le -\frac{3}{2} \alpha  H_{\nu}(\rho_t) + \error_{\mgf}^2 + \alpha \, H_{\nu}(\rho_t) \qquad \text{\footnotesize by $\alpha$-LSI} \\
    &= -\frac{1}{2} \alpha  H_{\nu}(\rho_t) +  \error_{\mgf}^2 .
\end{align*}
This is equivalent to 
\[\frac{\partial}{\partial t} e^{\frac{1}{2} \alpha t}  H_{\nu}(\rho_t) \le  e^{\frac{1}{2} \alpha t}  \error_{\mgf}^2.\]
Integrating from $0$ to $t$, we have
\[H_{\nu}(\rho_t) \le e^{-\frac{1}{2}\alpha t} H_{\nu}(\rho_0) + \frac{2}{\alpha} \, (1 - e^{-\frac{1}{2}\alpha t}) \, \error_{\mgf}^2.\]
\end{proof}

\section{Proof of Theorem~\ref{thm:kl-bounded-MGF}}
\label{sec:pf-kl-bdd-MGF}

\subsection{Proof of Lemma \ref{lemma-time-deriv-kl}}
\begin{proof}[Proof of Lemma \ref{lemma-time-deriv-kl}]
The continuity equation corresponding to Eq.~\eqref{eq:interpolation-SDE} is 
\[\frac{\partial \rho_t (x)}{\partial t} = -\nabla \cdot \left(\rho_t(x) \E_{\rho_{0 \vert t}}[s(x_0) \vert x_t = x] \right) + \Delta \rho_t (x).\]
Therefore,
\begin{align*}
    \frac{\partial}{\partial t}H_{\nu}{(\rho_t)} 
    &= \int_{\R^d} \left(-\nabla \cdot \left(\rho_t \, \E_{\rho_{0 \vert t}}[s(x_0) \vert x_t = x] \right) + \nabla \cdot (\rho_t \nabla \log \frac{\rho_t}{\nu}) + \nabla \cdot (\rho_t \nabla \log \nu)\right) \log \frac{\rho_t}{\nu} dx \\
    &= \int_{\R^d} \left( \nabla \cdot \left(\rho_t (\nabla \log \frac{\rho_t}{\nu} - \E_{\rho_{0 \vert t}}[s(x_0) \vert x_t = x] + \nabla \log \nu)\right)\right)\log \frac{\rho_t}{\nu} dx \\
    &= -\int_{\R^d}  \rho_t \langle \nabla \log \frac{\rho_t}{\nu} - \E_{\rho_{0 \vert t}}[s(x_0) \vert x_t = x] + \nabla \log \nu, \nabla \log \frac{\rho_t}{\nu} \rangle dx \qquad \text{\footnotesize by integration by parts}\\
    &= -\int_{\R^d}  \rho_t \| \nabla \log \frac{\rho_t}{\nu} \|^2 dx + \int_{\R^d} \rho_t \langle \E_{\rho_{0 \vert t}}[s(x_0) \vert x_t = x] - \nabla \log \nu, \nabla \log \frac{\rho_t}{\nu} \rangle dx \\
    &= -J_{\nu}(\rho_t) + \int_{\R^d} \rho_t \langle \E_{\rho_{0 \vert t}}[s(x_0) \vert x_t = x] - \nabla \log \nu, \nabla \log \frac{\rho_t}{\nu} \rangle dx \\
    &= -J_{\nu}(\rho_t) + \E_{\rho_{0t}}\left[\langle s(x_0) - \nabla \log \nu(x_t), \nabla \log \frac{\rho_t (x_t)}{\nu (x_t)}\rangle \right] \qquad \text{\footnotesize by renaming $x$ as $x_t$} \\
    &\leq -J_{\nu}(\rho_t) + \E_{\rho_{0t}}\left[ \| s(x_0) - \nabla \log \nu(x_t) \|^2 \right] + \frac{1}{4} \E_{\rho_{0t}}\left[ \| \nabla \log \frac{\rho_t (x_t)}{\nu (x_t)} \|^2 \right] \\
    &= -J_{\nu}(\rho_t) + \E_{\rho_{0t}}\left[ \| s(x_0) - \nabla \log \nu(x_t) \|^2 \right] + \frac{1}{4} J_{\nu}(\rho_t) \\
    &= -\frac{3}{4}J_{\nu}(\rho_t) + \E_{\rho_{0t}}\left[ \| s(x_0) - \nabla \log \nu(x_t) \|^2 \right].
\end{align*}
\end{proof}

\subsection{Proof of Lemma \ref{score-lipschitz}}
\begin{proof}[Proof of Lemma \ref{score-lipschitz}]
By $L_s$-Lipschitzness of $s$,
\[ \| s(x_t) - s(x_0) \|^2 \leq L_s^2 \| x_t - x_0 \|^2 = L_s^2 \| ts(x_0) + \sqrt{2t} z_0\|^2 \leq 2 L_s^2 t^2 \| s(x_0) \|^2 + 4L_s^2 t \| z_0\|^2.\]
For the sake of our subsequent analysis, we use a bound in terms of $s(x_t)$ rather than $s(x_0)$. Therefore, we opt to utilize 
\[ L_s \| x_t - x_0 \| \geq \| s(x_t) - s(x_0) \| \geq \|  s(x_0) \| - \| s(x_t) \|.\]
Rearranging it gives
\begin{align*}
    \|  s(x_0) \| &\leq L_s \| x_t - x_0 \| + \| s(x_t) \| \\
    &= L_s \| t s(x_0) + \sqrt{2t}\, z_0 \| + \| s(x_t) \| && \text{\footnotesize since $x_t=x_0 + t s(x_0) + \sqrt{2t} \,z_0 $}\\
    &= L_s t \| s(x_0) \| + L_s \sqrt{2t} \| z_0 \| + \| s(x_t) \| && \text{\footnotesize by triangle inequality} \\
    &\leq \frac{1}{3} \| s(x_0) \| + L_s \sqrt{2t} \| z_0 \| + \| s(x_t) \| && \text{\footnotesize since $t \leq \frac{1}{3L_s}$}.
\end{align*}
It follows that
\begin{equation}
\label{eq:bound1}
     \| s(x_0) \| \leq \frac{3}{2} \| s(x_t) \| + \frac{3}{\sqrt{2}} L_s \sqrt{t} \| z_0 \| 
    \implies  \| s(x_0) \|^2 \leq \frac{9}{2} \| s(x_t) \|^2 + 9 L_s^2 t \| z_0 \|^2  .
\end{equation}
So we can bound $\| s(x_t) - s(x_0) \|^2$ as follows
\begin{align*}
    \| s(x_t) - s(x_0) \|^2 &\leq 2 L_s^2 t^2 \| s(x_0) \|^2 + 4L_s^2 t \| z_0\|^2 \\
    &\leq 2 L_s^2 t^2 \left( \frac{9}{2} \| s(x_t) \|^2 + 9 L_s^2 t \| z_0 \|^2\right) + 4L_s^2 t \| z_0\|^2  \qquad \text{\footnotesize by plugging in Eq.~\eqref{eq:bound1}}\\
    &= 9 L_s^2 t^2 \| s(x_t) \|^2 + (18 L_s^4 t^3 + 4L_s^2 t) \| z_0 \|^2 \\
    &\leq 9 L_s^2 t^2 \| s(x_t) \|^2 + 6L_s^2 t \| z_0 \|^2  \qquad \qquad \text{\footnotesize since $t \leq \frac{1}{3L_s}$} \\
    &= 9 L_s^2 t^2 \| s(x_t)  - \nabla \log \nu (x_t) + \nabla \log \nu (x_t) \|^2 + 6L_s^2 t \| z_0 \|^2 \\
    &\leq 18 L_s^2 t^2 \| s(x_t)  - \nabla \log \nu (x_t) \|^2 + 18 L_s^2 t^2 \| \nabla \log \nu(x_t) \|^2 + 6L_s^2 t \| z_0 \|^2.
\end{align*}
\end{proof}

\subsection{Proof of Lemma \ref{lemma-one-step-conv-mgf}}
\begin{proof}[Proof of Lemma \ref{lemma-one-step-conv-mgf}]
Let $M(x) = \| \nabla \log \nu(x) - s(x) \|^2$. By Lemma \ref{lemma-time-deriv-kl},
\begin{align*}
    \frac{\partial}{\partial t}H_{\nu}{(\rho_t)}
    &\le -\frac{3}{4} J_{\nu}(\rho_t) + \E_{\rho_{0t}}\left[ \| s(x_0) - \nabla \log \nu(x_t) \|^2 \right] \\
    &\le -\frac{3}{4}J_{\nu}(\rho_t) + 2 \E_{\rho_{0t}}\left[ \| s(x_0) - s(x_t) \|^2\right] + 2 \E_{\rho_{t}}\left[ \| s(x_t) - \nabla \log \nu(x_t) \|^2\right] \\
    &\overset{(i)}{\le} -\frac{3}{4}J_{\nu}(\rho_t) + 2 \E_{\rho_{0t}}\left[ 18 L_s^2 t^2  M(x_t) + 18 L_s^2 t^2 \| \nabla \log \nu(x_t) \|^2 + 6L_s^2 t \| z_0 \|^2 \right] + 2 \E_{\rho_{t}}[ M(x)] \\
    &= -\frac{3}{4}J_{\nu}(\rho_t) + \left( 36 L_s^2 t^2 + 2 \right) \E_{\rho_{t}}[M(x)] +  36 L_s^2 t^2  \E_{\rho_{t}}\left[\| \nabla \log \nu(x) \|^2 \right]+ 12 d L_s^2 t \\
    &\le -\frac{3}{4}J_{\nu}(\rho_t) + \frac{9}{4} \E_{\rho_{t}}[M(x)] +  36 L_s^2 t^2  \E_{\rho_{t}}\left[\| \nabla \log \nu(x) \|^2 \right]+ 12 d L_s^2 t \quad \text{\footnotesize since $t^2 \leq h^2 \leq \frac{\alpha^2}{144L_s^2 L^2} \leq \frac{1}{144 L_s^2}$} \\
    &\overset{(ii)}{\le} -\frac{3}{4}J_{\nu}(\rho_t) + \frac{9}{4} \E_{\rho_{t}}[M(x)] +  36 L_s^2 t^2 \left( \frac{4L^2}{\alpha} H_{\nu}(\rho_t) + 2dL\right)+ 12 d L_s^2 t\\
    &= -\frac{3}{4}J_{\nu}(\rho_t) + \frac{9}{4} \E_{\rho_{t}}[M(x)] + \frac{ 144 L_s^2 t^2  L^2 }{\alpha} H_{\nu}(\rho_t) +  72d L_s^2 t^2 L+ 12 d L_s^2 t \\
    &\le -\frac{3}{4}J_{\nu}(\rho_t) + \frac{9}{4} \E_{\rho_{t}}[M(x)] + \alpha H_{\nu}(\rho_t) +  72 d L_s^2 t^2 L+ 12 d L_s^2 t \quad \text{\footnotesize since $t^2 \leq h^2 \leq \frac{\alpha^2}{144L_s^2 L^2}$} \\
    &\le -\frac{1}{2}\alpha H_{\nu}(\rho_t) + \frac{9}{4} \E_{\rho_{t}}[M(x)] + 72 d L_s^2 t^2 L+ 12 d L_s^2 t \quad \text{\footnotesize by $\alpha$-LSI}\\
\end{align*}
where $(i)$ is by Lemma \ref{score-lipschitz} where the condition $t \le \frac{1}{3L_s}$ holds since $t \le h \le \frac{\alpha}{12L_sL}$ and $\alpha < L$. $(ii)$ is by \cite[Lemma 12]{VW19} since $\nu$ is $\alpha$-LSI and $L$-smooth.
By the change of measure argument in Eq.~\eqref{eq:dv-kl}, the second term can be bounded as follows
\begin{align*}
    \E_{\rho_{t}}[ M(x)] \leq \error_{\mgf}^2 + \frac{\alpha}{9} H_{\nu}(\rho_t).
\end{align*}
So we have
\begin{align*}
    \frac{\partial}{\partial t} H_{\nu}{(\rho_t)} & \le -\frac{1}{4}\alpha H_{\nu}(\rho_t) + 72 d L_s^2 t^2 L+ 12 d L_s^2 t + \frac{9}{4} \error_{\mgf}^2 \\
    &\leq -\frac{1}{4}\alpha H_{\nu}(\rho_t) + 72 d L_s^2 h^2 L+ 12 d L_s^2 h + \frac{9}{4} \error_{\mgf}^2\quad \text{ \footnotesize since $t \in (0, h)$}.
\end{align*}
This is equivalent to 
\[\frac{\partial}{\partial t} e^{\frac{1}{4}\alpha t} H_{\nu}{(\rho_t)} \leq e^{\frac{1}{4}\alpha t} \left( 72 d L_s^2 h^2 L+ 12 d L_s^2 h + \frac{9}{4} \error_{\mgf}^2\right).\]
Hence,
\begin{align*}
    H_{\nu}{(\rho_h)} &\le e^{-\frac{1}{4}\alpha h} H_{\nu}{(\rho_0)} + e^{-\frac{1}{4}\alpha h} \, \frac{4 ( e^{\frac{1}{4}\alpha h}-1)}{\alpha} \left( 72 d L_s^2 h^2 L+ 12 d L_s^2 h + \frac{9}{4} \error_{\mgf}^2 \right) \\
    &\le e^{-\frac{1}{4}\alpha h} H_{\nu}{(\rho_0)} + 2h\left( 72 d L_s^2 h^2 L+ 12 d L_s^2 h + \frac{9}{4} \error_{\mgf}^2\right)
\end{align*}
where the last inequality uses $e^{-\frac{1}{4}\alpha h} \leq 1$ and $e^c - 1 \leq 2c$ for $c=\frac{1}{4}\alpha \step \in (0,1)$, which is satisfied since $h < \frac{1}{2 \alpha}$. Renaming $\rho_0$ as $\rho_k$ and $\rho_h$ as $\rho_{k+1}$, we obtain the desired bound
\[ H_{\nu}(\rho_{k+1}) \leq e^{-\frac{1}{4}\alpha h} H_{\nu}{(\rho_k)}  + 144 d L_s^2 L h^3 + 24 d L_s^2 h^2 +\frac{9}{2} \error_{\mgf}^2 h. \]
\end{proof}

\subsection{Proof of Theorem \ref{thm:kl-bounded-MGF}}
\label{sec:pf-kl-mgf}
Theorem \ref{thm:kl-bounded-MGF} directly follows from Lemma \ref{lemma-one-step-conv-mgf}.
\begin{proof}[Proof of Theorem \ref{thm:kl-bounded-MGF}]
Applying the recursion contraction in Lemma \ref{lemma-one-step-conv-mgf} $k$ times, we have
\begin{align*}
    H_{\nu}(\rho_{k}) &\leq e^{-\frac{1}{4}\alpha hk} H_{\nu}{(\rho_0)} + \sum_{i=0}^{k-1}e^{-\frac{1}{4}\alpha hi} \left(144 d L_s^2 L h^3+ 24 d L_s^2 h^2 +\frac{9}{2} \error_{\mgf}^2 \, h\right) \\
    &\overset{(i)}{\le} e^{-\frac{1}{4}\alpha hk} H_{\nu}{(\rho_0)} + \frac{1}{1 - e^{-\frac{1}{4}\alpha h}} \left(144 d L_s^2 L h^3 + 24 d L_s^2 h^2 + \frac{9}{2} \error_{\mgf}^2 \, h\right) \\
    &\le e^{-\frac{1}{4}\alpha hk} H_{\nu}{(\rho_0)} + \frac{16}{3 \alpha h} \left(144 d L_s^2 L h^3 + 24 d L_s^2 h^2 +  \frac{9}{2} \error_{\mgf}^2 \, h \right) \\
    &= e^{-\frac{1}{4}\alpha hk} H_{\nu}{(\rho_0)} + \frac{768 d L_s^2 L}{\alpha} h^2 + \frac{128 d L_s^2}{\alpha} h + \frac{8}{3 \alpha}\error_{\mgf}^2 \\
    &\le e^{-\frac{1}{4}\alpha hk} H_{\nu}{(\rho_0)} + \frac{128\, L_s}{\alpha} (L_s + L)\, d \,h + \frac{8}{3 \alpha}\error_{\mgf}^2 \quad \text{\footnotesize since $h < \frac{\alpha}{12 L_s L} \le \frac{1}{12 L_s}$}
\end{align*}
where $(i)$ uses the inequality $1-e^{-c} \ge \frac{3}{4}c$ for $0 < c = \frac{1}{4} \alpha h \le \frac{1}{4}$, which holds since $h \le \frac{1}{2} \alpha$.
\end{proof}

\section{Proof of Theorem~\ref{thm:renyi-max-error-bd}}
\label{sec:pf-renyi-max-error-bd}
First, we review the definition of R\'enyi information.
\begin{definition}[R\'enyi information]
For $q\geq 0$, the R\'enyi information of order $q$ of $\rho$ w.r.t. $\nu$ is
\[ G_{q, \nu}(\rho) =  \E_{\nu}\left[ \left( \frac{\rho}{\nu}\right)^q \big\| \nabla \log \frac{\rho}{\nu}\big\|^2\right] = \frac{4}{q^2} \E_{\nu}\left[\Big\| \nabla \left( \frac{\rho}{\nu} \right)^{\frac{q}{2}}\Big\|^2\right].\]
\end{definition}
When $q=1$, R\'enyi information recovers the relative Fisher information. 

For the reader's convenience, we restate the full theorem here.
\begin{reptheorem}{thm:renyi-max-error-bd}
Assume $\nu$ is $\alpha$-LSI and $f=-\log \nu$ is $L$-smooth, and score estimator $s$ is $L_s$-Lipschitz and has finite $L^\infty$ error. Let $q \geq 1$. If $\,0 < h < \min(\frac{\alpha}{12 L L_s q}, \frac{q}{4\alpha})$, then after $k$ iterations of ILA~\eqref{Eq:ILA},
\[ R_{q, \nu}(\rho_{k}) \leq e^{-\frac{\alpha hk}{q}}  R_{q, \nu}(\rho_{0}) + C_1 d h + ( C_2 h^2 + C_3)\, \error_{\infty}^2,\]
where $C_1=\frac{16 L_s q}{\alpha}(L + 2L_s q)$ and $C_2=\frac{96 L_s^2 q^2}{\alpha}$ and $C_3 = \frac{16 q^2}{3\alpha}$.
\end{reptheorem}
To prove Theorem~\ref{thm:renyi-max-error-bd}, we first show the following auxiliary results. The proof of Theorem~\ref{thm:renyi-max-error-bd} is in Appendix~\ref{sec:pf-renyi-max-error}.
\begin{lemma}
\label{time-deriv-renyi}
Let $\varphi_t = \frac{\rho_t}{\nu}$ and $\psi_t = \frac{\varphi_t^{q-1}}{\E_{\nu}[\varphi_t^q]} = \frac{\varphi_t^{q-1}}{F_{q, \nu}(\rho_t)}$. Then we have the following bound for the time derivative of R\'enyi divergence,
\[\frac{\partial}{\partial t} R_{q, \nu}(\rho_t) \leq -\frac{3}{4} q \frac{G_{q, \nu}(\rho_t)}{F_{q, \nu}(\rho_t)} + q \E_{\rho_{0t}} \left[ \psi_t(x_t) \| s(x_0) - \nabla \log \nu (x_t) \|^2 \right].\]
\end{lemma}
This is a generalized version of \cite[Proposition 15]{CEL+22} to the setting of estimated score. 
\begin{lemma}
\label{lemma-one-step-renyi}
Suppose the assumptions in Theorem \ref{thm:renyi-max-error-bd} hold, then along each step of ILA~\eqref{Eq:ILA}, we have
\[R_{q, \nu}(\rho_{k+1}) \leq e^{-\frac{\alpha}{q}h} R_{q, \nu} (\rho_k) + 144 L_s^2 dLq  h^3 + 24L_s^2 dqh^2 +\left( 72L_s^2h^3 + 4h\right) \error_{\infty}^2q.\]
\end{lemma}

\subsection{Proof of Lemma \ref{time-deriv-renyi}}
\begin{proof}[Proof of Lemma \ref{time-deriv-renyi}]
\begin{align*}
    \frac{\partial }{\partial t} R_{q, \nu}(\rho_{t}) &= \frac{1}{q-1} \frac{\int \frac{\frac{\partial}{\partial t} \rho_{t}^{q}}{\nu^{q-1}} dx}{F_{q, \nu}(\rho_{t})} \\
    &= \frac{q}{q-1} \frac{\int \left(\frac{\rho_{t}}{\nu}\right)^{q-1}\frac{\partial \rho_{t}}{\partial t}  dx}{F_{q, \nu}(\rho_{t})} \\
    &= \frac{q}{(q-1)F_{q, \nu}(\rho_{t})}\int \left(\frac{\rho_{t}}{\nu}\right)^{q-1}\frac{\partial \rho_{t}}{\partial t} dx \\
    &= \frac{q}{(q-1)F_{q, \nu}(\rho_{t})} \int \left(\frac{\rho_{t}}{\nu}\right)^{q-1} \left(-\nabla \cdot \left(\rho_{t} \E_{\rho_{0 \vert t}}[s(x_{0}) \vert x_{t} = x] \right) + \Delta \rho_{t} \right) dx \\
    &= \frac{q}{(q-1)F_{q, \nu}(\rho_{t})} \int -\rho_{t} \left\langle \nabla \left(\frac{\rho_{t}}{\nu}\right)^{q-1},  \nabla \log \frac{\rho_{t}}{\nu} - \E_{\rho_{0 \vert t}}[s(x_{0}) \vert x_{t} = x] + \nabla \log \nu \right\rangle dx \\
    &= \frac{q}{(q-1)F_{q, \nu}(\rho_{t})} \Big( \int -\rho_t \left\langle \nabla \left(\frac{\rho_{t}}{\nu}\right)^{q-1},  \nabla \log \frac{\rho_{t}}{\nu}\right\rangle dx \\
    &\qquad + \int \rho_{t} \left\langle \nabla \left(\frac{\rho_{t}}{\nu}\right)^{q-1}, \E_{\rho_{0 \vert t}}[s(x_{0}) \vert x_{t} = x] - \nabla \log \nu \right\rangle dx \Big) \\
    &= \frac{q}{(q-1)F_{q, \nu}(\rho_{t})} \Big( - \underbrace{\int  \nu \left\langle \nabla \left(\frac{\rho_{t}}{\nu}\right)^{q-1},  \nabla \frac{\rho_{t}}{\nu}\right\rangle dx}_{A_1} \\
    & \qquad + \underbrace{\int \rho_{t} \left\langle \nabla \left(\frac{\rho_{t}}{\nu}\right)^{q-1}, \E_{\rho_{0 \vert t}}[s(x_{0}) \vert x_{t} = x] - \nabla \log \nu \right\rangle dx}_{A_2}\Big)
\end{align*}
By noting that
\begin{align*}
    \left\langle \nabla \left(\frac{\rho_{t}}{\nu}\right)^{q-1},  \nabla \frac{\rho_{t}}{\nu}\right\rangle &= (q-1)\left\langle \left(\frac{\rho_{t}}{\nu}\right)^{q-2}\nabla \frac{\rho_{t}}{\nu},\nabla \frac{\rho_{t}}{\nu} \right\rangle \\
    &= (q-1)\left\langle \left(\frac{\rho_{t}}{\nu}\right)^{\frac{q-2}{2}}\nabla \frac{\rho_{t}}{\nu}, \left(\frac{\rho_{t}}{\nu}\right)^{\frac{q-2}{2}}\nabla \frac{\rho_{t}}{\nu} \right\rangle \\
    &= (q-1) \left\| \frac{2}{q} \nabla \left(\frac{\rho_{t}}{\nu} \right)^{\frac{q}{2}}\right\|^2 \\
    &= \frac{4(q-1)}{q^2} \left\| \nabla \left(\frac{\rho_{t}}{\nu} \right)^{\frac{q}{2}}\right\|^2,
\end{align*}
we obtain $A_1 = \dfrac{4(q-1)}{q^2} \E_{\nu}\left[ \left\| \nabla \left(\dfrac{\rho_{t}}{\nu} \right)^{\frac{q}{2}}\right\|^2\right].$

On the other hand, since $\nabla \left(\frac{\rho_{t}}{\nu} \right)^{q-1} = (q-1) \left(\frac{\rho_{t}}{\nu} \right)^{q-2} \nabla \frac{\rho_{t}}{\nu} = (q-1) \left(\frac{\rho_{t}}{\nu} \right)^{\frac{q-2}{2}} \left(\frac{\rho_{t}}{\nu} \right)^{\frac{q-2}{2}} \nabla \frac{\rho_{t}}{\nu} = (q-1) \left(\frac{\rho_{t}}{\nu} \right)^{\frac{q-2}{2}} \frac{2}{q} \nabla \left(\frac{\rho_{t}}{\nu} \right)^{q/2}$, we have
\begin{align*}
    A_2 &= \int \rho_t \left\langle \nabla \left(\frac{\rho_{t}}{\nu}\right)^{q-1}, \E_{\rho_{0 \vert t}}[s(x_0) \vert x_t = x] - \nabla \log \nu \right \rangle dx \\
    &= 2 \frac{q-1}{q} \E_{\rho_{0t}}\left[ \left(\frac{\rho_{t}}{\nu} \right)^{\frac{q-2}{2}} \left\langle \nabla \left(\frac{\rho_{t}}{\nu} \right)^{\frac{q}{2}} , s(x_0) - \nabla \log \nu(x_t) \right\rangle \right] \\
    &= 2 \frac{q-1}{q} \E_{\rho_{0t}}\left[ \big\langle \left(\frac{\rho_{t}}{\nu} \right)^{-\frac{1}{2}} \nabla \left(\frac{\rho_{t}}{\nu} \right)^{\frac{q}{2}} , \left(\frac{\rho_{t}}{\nu} \right)^{\frac{q-1}{2}} \left(s(x_0) - \nabla \log \nu(x_t) \right) \big\rangle \right].
\end{align*}
Applying $\langle x,y \rangle \le \frac{1}{2q} \| x \|^2 + \frac{q}{2} \| y \|^2$, we obtain
\begin{align*}
    & \big\langle \left(\frac{\rho_{t}}{\nu} \right)^{-\frac{1}{2}} \nabla \left(\frac{\rho_{t}}{\nu} \right)^{\frac{q}{2}} , \left(\frac{\rho_{t}}{\nu} \right)^{\frac{q-1}{2}} \left(s(x_0) - \nabla \log \nu(x_t) \right) \big\rangle \\
    \leq & \; \frac{1}{2q} \| \left(\frac{\rho_{t}}{\nu} \right)^{-\frac{1}{2}} \nabla \left(\frac{\rho_{t}}{\nu} \right)^{\frac{q}{2}} \|^2 + \frac{q}{2} \| \left(\frac{\rho_{t}}{\nu} \right)^{\frac{q-1}{2}} \left(s(x_0) - \nabla \log \nu (x_t)\right) \|^2 \\
    = & \;\frac{1}{2q} \frac{\nu}{\rho_t} \| \nabla \left( \frac{\rho_t}{\nu}\right)^{\frac{q}{2}} \|^2 + \frac{q}{2} \left(\frac{\rho_t}{\nu} \right)^{q-1} \| s(x_0) - \nabla \log \nu (x_t) \|^2.
\end{align*}
Therefore,
\begin{align*}
    A_2 &\leq 2 \frac{q-1}{q} \left( \frac{1}{2q} \E_{\nu}\left[\| \nabla \left( \frac{\rho_t}{\nu}\right)^{\frac{q}{2}} \|^2\right] + \frac{q}{2} \E_{\rho_{0t}} \left[ \left(\frac{\rho_t}{\nu} \right)^{q-1} \| s(x_0) - \nabla \log \nu (x_t) \|^2 \right] \right) \\
    &=  \frac{q-1}{q^2} \E_{\nu}\left[\| \nabla \left( \frac{\rho_t}{\nu}\right)^{\frac{q}{2}} \|^2\right] + (q-1) \E_{\rho_{0t}} \left[ \left(\frac{\rho_t}{\nu} \right)^{q-1} \| s(x_0) - \nabla \log \nu (x_t) \|^2 \right].
\end{align*}
Hence,
\begin{align*}
    \frac{\partial}{\partial t} R_{q, \nu}(\rho_t) &= \frac{q}{(q-1)F_{q, \nu}(\rho_t)} ( - A_1 + A_2 ) \\
    &\leq \frac{q}{(q-1)F_{q, \nu}(\rho_t)} \left( -\frac{3(q-1)}{q^2} \E_{\nu}\left[\| \nabla \left( \frac{\rho_t}{\nu}\right)^{\frac{q}{2}} \|^2\right] + (q-1) \E_{\rho_{0t}} \left[ \left(\frac{\rho_t}{\nu} \right)^{q-1} \| s(x_0) - \nabla \log \nu (x_t) \|^2 \right] \right) \\
    &= - \frac{1}{F_{q, \nu}(\rho_t)} \left( \frac{3}{q}\E_{\nu}\left[\| \nabla \left( \frac{\rho_t}{\nu}\right)^{\frac{q}{2}} \|^2\right] - q \E_{\rho_{0t}} \left[ \left(\frac{\rho_t}{\nu} \right)^{q-1} \| s(x_0) - \nabla \log \nu (x_t) \|^2 \right] \right).
\end{align*}
Let $\varphi_t = \frac{\rho_t}{\nu}$ and $\psi_t = \frac{\varphi_t^{q-1}}{\E_{\nu}[\varphi_t^q]} = \frac{\varphi_t^{q-1}}{F_{q, \nu}(\rho_t)}$. We obtain the desired bound
\[ \frac{\partial}{\partial t} R_{q, \nu}(\rho_t) \leq -\frac{3}{4} q \frac{G_{q, \nu}(\rho_t)}{F_{q, \nu}(\rho_t)} + q \E_{\rho_{0t}} \left[ \psi_t(x_t) \| s(x_0) - \nabla \log \nu (x_t) \|^2 \right].\]
\end{proof}

\subsection{Proof of Lemma \ref{lemma-one-step-renyi}}
\begin{proof}[Proof of Lemma \ref{lemma-one-step-renyi}] Following
Lemma \ref{time-deriv-renyi}, we have
\begin{align}
\frac{\partial}{\partial t} R_{q, \nu}(\rho_t) & \le -\frac{3}{4} q \frac{G_{q, \nu}(\rho_t)}{F_{q, \nu}(\rho_t)} + q \E_{\rho_{0t}} \left[ \psi_t(x_t) \| s(x_0) - \nabla \log \nu (x_t) \|^2 \right] \notag\\
&\le -\frac{3}{4} q \frac{G_{q, \nu}(\rho_t)}{F_{q, \nu}(\rho_t)} + 2q\underbrace{\E_{\rho_{0t}} \left[  \psi_t(x_t) \| s(x_0) - s(x_t) \|^2 \right]}_{A_3} + 2q \E_{\rho_{t}} \left[ \psi_t(x) \| s(x) - \nabla \log \nu(x) \|^2 \right]. \label{eq:lem8eq1}
\end{align}
And $A_3$ can be bounded as follows
\begin{align*}
    A_3 &\leq \E_{\rho_{0t}} \left[ \psi_t(x_t) \left( 18 L_s^2 t^2 \| s(x_t)  - \nabla \log \nu (x_t) \|^2 + 18 L_s^2 t^2 \| \nabla \log \nu(x_t) \|^2 + 6L_s^2 t \| z_0 \|^2 \right) \right] \qquad \text{\footnotesize by Lemma \ref{score-lipschitz}} \\
    &=18 L_s^2 t^2  \E_{\rho_{t}} \left[ \psi_t(x) \| s(x)  - \nabla \log \nu (x) \|^2 \right] + 18 L_s^2 t^2 \E_{\rho_{t}} \left[ \psi_t(x) \| \nabla \log \nu(x) \|^2 \right] + 6L_s^2 t d \\
    &=18 L_s^2 t^2  \E_{\rho_{t}\psi_t} \left[ \| s(x)  - \nabla \log \nu (x) \|^2 \right] + 18 L_s^2 t^2 \E_{\rho_{t}\psi_t} \left[ \| \nabla \log \nu(x) \|^2 \right]  + 6L_s^2 t d.
\end{align*}
So we have
\begin{align}
\label{eq:lem8eq2}
\E_{\rho_{0t}} \left[ \psi_t(x_t) \| s(x_0) - \nabla \log \nu (x_t) \|^2 \right] \le (36 L_s^2 t^2 +2)\error_{\infty}^2 + 36 L_s^2 t^2 \underbrace{\E_{\rho_{t}\psi_t} \left[ \| \nabla \log \nu(x) \|^2 \right]}_{A_4}  + 12L_s^2 t d.
\end{align}
By~\cite[Lemma 16]{CEL+22} under the assumption of $\nabla \log \nu$ being $L$-Lipschitz,
\begin{align}
    A_4 &\leq \E_{\rho_{t}\psi_t} \left[ \big\| \nabla \log \frac{\rho_t \psi_t }{\nu} \big\|^2 \right] + 2dL \notag  \\
    &= \E_{\rho_{t}\psi_t} \left[ \big\|  \frac{\nu}{\rho_t \psi_t } \nabla\frac{\rho_t \psi_t }{\nu} \big\|^2 \right] + 2dL \notag \\
    &= \E_{\rho_{t}\psi_t} \left[ \big\|  \frac{\nu}{\rho_t \psi_t } \frac{1}{F_{q,\nu}(\rho_t)}\nabla\varphi_t^q \big\|^2 \right] + 2dL \notag \\
    &= \int \frac{\nu^2}{\rho_t \psi_t F_{q, \nu}^2 (\rho_t)}  \big \| \nabla\varphi_t^q \big\|^2 dx + 2dL \notag \\
    &= \frac{\E_{\nu}\left[ \frac{1}{\varphi_t^q} \big \| \nabla\varphi_t^q \big\|^2\right]}{F_{q, \nu}(\rho_t)} + 2dL \notag \\
    &= \frac{ 4 \E_{\nu}\left[ \big \| \nabla\varphi_t^{\frac{q}{2}} \big\|^2\right]}{F_{q, \nu}(\rho_t)} + 2dL \qquad \text{\footnotesize by $\frac{1}{\varphi_t^q} \big \| \nabla\varphi_t^q \big\|^2 = 4\big \| \nabla\varphi_t^{\frac{q}{2}} \big\|^2$} \notag \\
    &= q^2 \frac{G_{q, \nu}(\rho_t)}{F_{q, \nu}(\rho_t)} + 2dL.\label{eq:lem8eq3}
\end{align}
Combing Eq.~\eqref{eq:lem8eq1}-~\eqref{eq:lem8eq3}, we have
\begin{align*}
    \frac{\partial}{\partial t} R_{q, \nu} (\rho_t) &\leq \left(36L_s^2 t^2 q^3 - \frac{3}{4}q \right) \frac{G_{q, \nu}(\rho_t)}{F_{q, \nu}(\rho_t)} + \left( 36L_s^2t^2 + 2\right) \error_{\infty}^2q + 72L_s^2 t^2 dLq + 12L_s^2 tdq \\
    &\leq - \frac{1}{2}q \frac{G_{q, \nu}(\rho_t)}{F_{q, \nu}(\rho_t)} + \left( 36L_s^2h^2 + 2\right) \error_{\infty}^2q + 72L_s^2  dLq h^2+ 12L_s^2 dqh \qquad \text{\footnotesize since $t^2 \leq h^2 \leq \frac{\alpha^2}{144L_s^2 q^2 L^2}$} \\
    &\leq - \frac{\alpha}{q} R_{q, \nu}(\rho_t) + \left( 36L_s^2h^2 + 2\right) \error_{\infty}^2q + 72L_s^2 dLq h^2 + 12L_s^2 dqh
\end{align*}
where the last inequality is from \cite[Lemma 5]{VW19} under the assumption of $\nu$ satisfying $\alpha$-LSI. It follows that 
\[  \frac{\partial}{\partial t} e^{\frac{\alpha}{q}t}R_{q, \nu} (\rho_t) \leq e^{\frac{\alpha}{q}t}\left( 72L_s^2 dLq  h^2 + 12L_s^2 dqh +\left( 36L_s^2h^2 + 2\right) \error_{\infty}^2q \right).\]
Integrating from $0$ to $h$, we have
\begin{align*}
    e^{\frac{\alpha}{q}h}R_{q, \nu} (\rho_h) &\leq R_{q, \nu} (\rho_0) + \frac{q (e^{\frac{\alpha}{q}h}-1)}{\alpha}\left(72L_s^2 dLq  h^2 + 12L_s^2 dqh +\left( 36L_s^2h^2 + 2\right) \error_{\infty}^2q \right) \\
    &\leq R_{q, \nu} (\rho_0) + 2h\left(72L_s^2 dLq  h^2 + 12L_s^2 dqh +\left( 36L_s^2h^2 + 2\right) \error_{\infty}^2q \right).
\end{align*}
The last inequality uses $e^c-1 \leq 2c$ for $c=\frac{\alpha}{q}h \in (0, 1)$. Rearranging and renaming $\rho_0 \equiv \rho_k, \rho_h \equiv \rho_{k+1}$, we obtain the desired recursive contraction
\[ R_{q, \nu} (\rho_{k+1}) \leq e^{-\frac{\alpha}{q}h} R_{q, \nu} (\rho_k) + 144 L_s^2 dLq  h^3 + 24L_s^2 dqh^2 +\left( 72L_s^2h^3 + 4h\right) \error_{\infty}^2q.\]
\end{proof}

\subsection{Proof of Theorem \ref{thm:renyi-max-error-bd}}
\label{sec:pf-renyi-max-error}
\begin{proof}[Proof of Theorem \ref{thm:renyi-max-error-bd}]
Applying Lemma \ref{lemma-one-step-renyi} $k$ times, we have
\begin{align*}
    R_{q, \nu} (\rho_{k}) &\leq e^{-\frac{\alpha hk}{q}} R_{q, \nu} (\rho_0) + \sum_{i=1}^{k-1}e^{-\frac{\alpha h}{q}i} \left(144 L_s^2 dLq  h^3 + 24L_s^2 dqh^2 +\left( 72L_s^2h^3 + 4h\right) \error_{\infty}^2q \right) \\
    &\leq e^{-\frac{\alpha hk}{q}} R_{q, \nu} (\rho_0) + \frac{1}{1-e^{-\frac{\alpha h}{q}}} \left(144 L_s^2 dLq  h^3 + 24L_s^2 dqh^2 +\left( 72L_s^2h^3 + 4h\right) \error_{\infty}^2q \right) \\
    &\overset{(i)}{\le} e^{-\frac{\alpha hk}{q}} R_{q, \nu} (\rho_0) + \frac{4q}{3\alpha h} \left(144 L_s^2 dLq  h^3 + 24L_s^2 dqh^2 +\left( 72L_s^2h^3 + 4h\right) \error_{\infty}^2q \right) \\
    &\leq e^{-\frac{\alpha hk}{q}}  R_{q, \nu}(\rho_{0}) + \frac{192 d L L_s^2 q^2}{ \alpha} h^2 + \frac{32 d L_s^2 q^2}{\alpha} h + \left(\frac{96 L_s^2 h^2 q^2}{\alpha} + \frac{16q^2}{3\alpha}\right) \error_{\infty}^2 \\
    &\le e^{-\frac{\alpha hk}{q}}  R_{q, \nu}(\rho_{0}) + \frac{16 d L_s q}{\alpha}(L + 2L_s q) h +\frac{q^2}{\alpha} \left( 96 L_s^2 h^2 + \frac{16}{3}\right) \error_{\infty}^2.
\end{align*}
where $(i)$ uses the inequality $1-e^{-c} \geq \frac{3}{4}c$ for $ 0 < c = \frac{\alpha h}{q} < \frac{1}{4}$, which holds since $h \le \frac{q}{4\alpha}$.
\end{proof}

\section{Convergence in R\'enyi divergence when the estimator is a score function}
\label{sec:renyi-est-scorefxc}

It is common for the estimator to actually be the score function of another distribution $\hat{\nu}$ which approximates $\nu$. Consequently, when employing the score estimator $s = s_{\hat \nu}$, ILD for $\nu$ becomes equivalent to running the exact Langevin dynamics for $\hat \nu$. Therefore, we can characterize the performance of ILD by considering the performance of Langevin dynamics for $\hat \nu$.

Suppose $\hat \nu$ satisfies LSI. By Fact~\ref{fact:renyi-langevin-conti} and Fact~\ref{fact:renyi-ula}, we know that R\'enyi divergence with respect to $\hat \nu$ converges exponentially fast for both Langevin dynamics and ULA targeting $\hat \nu$. Combining the two facts with the generalized triangle inequality for R\'enyi divergence, we obtain the following biased convergence rate of R\'enyi divergence with respect to $\nu$ for ILD and ILA.

\begin{proposition}[Convergence of R\'enyi divergence for ILD]
\label{prop:renyi-ild}
Suppose we estimate the score function of target distribution $\nu$ by that of another distribution $\hat{\nu}$, and $\hat{\nu}$ satisfies $\alpha$-LSI. Let $q \ge 1$. Assume $F_{2q-1, \nu}(\hat \nu) < \infty$. 
Then for $X_t \sim \rho_t$ along the ILD~\eqref{Eq:ILD} and $t \ge t_0 \coloneqq \frac{1}{2\alpha}\log (2q-1)$,
\[R_{q, \nu}(\rho_t) \le \frac{q-1/2}{q-1} e^{-\frac{\alpha (t-t_0)}{q}} R_{2, \hat \nu}(\rho_0) + R_{2q-1, \nu}(\hat \nu).\]
\end{proposition}

\begin{proof}[Proof of Proposition~\ref{prop:renyi-ild}]
\begin{align*}
    R_{q, \nu}(\rho_t) &\overset{(i)}{\le} \frac{q-1/2}{q-1} R_{2q, \hat{\nu}} (\rho_t) + R_{2q-1, \nu}(\hat{\nu}) \\
    &\overset{(ii)}{\le} \frac{q-1/2}{q-1} e^{-\frac{\alpha (t-t_0)}{q}} R_{2, \hat \nu}(\rho_0) + R_{2q-1, \nu}(\hat \nu),
\end{align*}
where $(i)$ is the decomposition of R\'enyi divergence~\cite[Lemma 7]{VW19} and $(ii)$ is from~\cite[Corollary 2]{VW19}.
\end{proof}

\begin{proposition}[Convergence of R\'enyi divergence for ILA]
\label{prop:renyi-ila}
Suppose we estimate the score function of target distribution $\nu$ by that of another distribution $\hat{\nu}$. Assume $\hat{\nu}$ is $\alpha$-LSI and $L$-smooth. Assume for simplicity that $\alpha \le 1 \le L$ and $q \ge 3$. If $F_{2q-1, \nu}(\hat \nu) < \infty$ and $\,0 < h < \frac{\alpha}{192q^2L^2}$, then after $k > K_0 = \frac{2}{\alpha h} \log \frac{q-1}{2}$ iterations of ILA~\eqref{Eq:ILA},
\[R_{q, \nu}(\rho_k) \le \frac{q-1/2}{q-1} \exp\left(-\frac{\alpha h (k-K_0)}{4}\right) R_{2, \hat \nu}(\rho_0) + R_{2q-1, \nu}(\hat \nu) + \tilde{O}\left( \frac{dhqL^2}{\alpha}\right).\]
\end{proposition}
\begin{proof}[Proof of Proposition~\ref{prop:renyi-ila}]
\begin{align*}
    R_{q, \nu}(\rho_k) &\overset{(i)}{\le} \frac{q-1/2}{q-1} R_{2q, \hat{\nu}} (\rho_k) + R_{2q-1, \nu}(\hat{\nu}) \\
    &\overset{(ii)}{\le} \frac{q-1/2}{q-1} \exp\left(-\frac{\alpha h (k-K_0)}{4}\right) R_{2, \hat \nu}(\rho_0) + R_{2q-1, \nu}(\hat \nu) + \tilde{O}\left( \frac{dhqL^2}{\alpha}\right),
\end{align*}
again $(i)$ is from~\cite[Lemma 7]{VW19}, and $(ii)$ is from~ \cite[Theorem 4]{CEL+22}.
\end{proof}
\paragraph{Comparison with Theorem~\ref{thm:renyi-max-error-bd}.}Theorem~\ref{thm:renyi-max-error-bd} requires the target distribution to satisfy LSI and smoothness, and for the score estimator to be Lipschitz with finite $\varepsilon_\infty$ error.
Proposition~\ref{prop:renyi-ila} does not impose any structural assumptions on the target; instead, it assumes the score estimator is the score of an approximate distribution $\hat \nu$, which satisfies LSI and smoothness.

We provide a simple example for comparison: Let the target distribution be $\nu = \N(0, \Sigma)$ on $\R^d$ with $\alpha I_d \preceq \Sigma^{-1} \preceq LI_d$. Suppose we estimate its score $s_\nu(x) = -\Sigma^{-1}x$ by\ $\hat s(x) = -\Sigma^{-1}(x-m)$ for some $m \in \R^d$, which is the score of $\hat{\nu} = \N(m, \Sigma)$. Then the $L^\infty$ error is $\varepsilon_\infty = \|\Sigma^{-1}m\| \le L \|m\|$, consequently the asymptotic bias from Theorem~\ref{thm:renyi-max-error-bd} (as $k \to \infty$ and $h \to 0$) is $O(\frac{q^2L^2 \|m\|^2}{\alpha})$. 
On the other hand, the asymptotic bias from Proposition~\ref{prop:renyi-ila} is $R_{2q-1, \nu}(\hat\nu) = (q-\frac{1}{2}) \|m\|_{\Sigma^{-1}}^2 \le qL\|m\|^2$, which is smaller than $\frac{q^2L^2 \|m\|^2}{\alpha}$ since $q \ge 1$, $\alpha \le L$.

We also note that the assumptions in Proposition~\ref{prop:renyi-ila} might be more applicable than Theorem~\ref{thm:renyi-max-error-bd}.
Our current work does not address the question of finding an approximate distribution $\hat \nu$ that satisfies LSI and smoothness. This is an interesting statistical problem that can be approached e.g.\ via variational inference $\arg\min_{\hat\nu \in \{\text{LSI, smooth}\}} R_{q,\nu}(\hat{\nu})$. We leave a detailed study of this problem for future work.

\section{Proof of Theorem~\ref{thm:conv-ddpm}}
\label{sec:pf-conv-ddpm}

To establish Theorem~\ref{thm:conv-ddpm}, we begin by formulating Lemma~\ref{lemma:one-step-ddpm}. The proof of this lemma relies on three auxiliary results, which we present next.
\begin{lemma}
\label{lemma:continu-eq-bou-score}
The continuity equation for the interpolation SDE~\eqref{eq:interp-sde} is
\[\frac{\partial \rho_{t}}{\partial t} = - \nabla \cdot (\rho_{t} y) - 2 \nabla \cdot \left( \rho_t \E_{\rho_{0 \mid t}}[\hat{s}_{T_0}(Y_0) \mid Y_t = y]\right) + \Delta\rho_t.\]
\end{lemma}

\begin{lemma}
\label{lemma:score-lipschitz}
Assume the score estimator $\hat{s}_{t'}$ is $L_s$-Lipschitz for $0 \le t' \le T$. If $t \le \frac{1}{12L_s}$, then
\[\| \hat{s}_{t'}(y_t) - \hat{s}_{t'}(y_0) \|^2 \le 36 t^2 L_s^2 \| y_0 \|^2 + 36 L_s^2 t \| z \|^2 + 72 L_s^2 t^2 \| \hat{s}_{t'}(y_t) \|^2\]
where $y_t = e^{t} y_0 + 2(e^{t}-  1) \hat{s}_{t'}(y_0) + \sqrt{e^{2 t} - 1}\, z$ and $z \sim \N(0, I_d)$.
\end{lemma}

\begin{lemma}\label{lemma:deriv-kl-ddpm}
Suppose the assumptions in Theorem \ref{thm:conv-ddpm} hold. Let $\mu_t$ be the distribution at time $t$ along the true backward process~\eqref{eq:backwardOU} and $\rho_t$ be the distribution at time $t$ along the interpolation SDE~\eqref{eq:interp-sde}. 
If $0 < \step \le \frac{1}{96 L_s L}$, then 
\[\frac{d}{dt} H_{\mu_t}(\rho_{t}) \le  -\frac{\alpha_{T_0 - t}}{4} H_{\mu_{t}}(\rho_t) + \frac{65}{8} \error_{\mgf}^2 + 9 L_s (3 + 32 L_s)\, dt. \]
\end{lemma}

\subsection{Proof of Lemma \ref{lemma:continu-eq-bou-score}}
\begin{proof}[Proof of Lemma \ref{lemma:continu-eq-bou-score}]
Conditioning on $y_0$, the Fokker-Planck equation for the conditional distribution $\rho_{t \mid 0}$ is
\[\frac{\partial \rho_{t \mid 0}(y_t \mid y_0)}{\partial t} = -\nabla \cdot \left(\rho_{t \mid 0}(y_t \mid y_0) (y_t + 2\hat{s}_{T_0}(y_0) )\right) + \Delta\rho_{t \mid 0}(y_t \mid y_0).\]
Therefore, we have
\begin{align*}
  \frac{\partial \rho_{t}(y_t)}{\partial t} &= \frac{\partial }{\partial t} \int \rho_{t \mid 0}(y_t \mid y_0) \rho_{0}(y_0) dy_0 \\
  &= \int \frac{\partial }{\partial t} \rho_{t \mid 0}(y_t \mid y_0) \rho_{0}(y_0) dy_0 \\
  &= \int \left[ -\nabla \cdot \left(\rho_{t \mid 0}(y_t \mid y_0) ( y_t + 2\hat{s}_{T_0}(y_0) )\right) + \Delta\rho_{t \mid 0}(y_t \mid y_0) \right] \rho_{0}(y_0) dy_0 \\
  &= \int  -\nabla \cdot \left(\rho_{t, 0}(y_t, y_0) ( y_t + 2\hat{s}_{T_0}(y_0) )\right) + \Delta\rho_{t, 0}(y_t, y_0) dy_0 \\
  &= \int  -\nabla \cdot \left(\rho_{t, 0}(y_t, y_0) y_t \right) dy_0 - \int  \nabla \cdot \left(2\rho_{t, 0}(y_t, y_0)\hat{s}_{T_0}(y_0) )\right) + \int \Delta\rho_{t, 0}(y_t, y_0) dy_0 \\
  &= -\nabla \cdot \left(\rho_{t}(y_t) y_t \right) - 2 \nabla \cdot \left( \rho_{t}(y) \E_{\rho_{0 \mid t}}[\hat{s}_{T_0}(y_0) \mid Y_t = y]\right) + \Delta \rho_t (y_t).
\end{align*}
\end{proof}

\subsection{Proof of Lemma \ref{lemma:score-lipschitz}}
\begin{proof}[Proof of Lemma \ref{lemma:score-lipschitz}]
\begin{align*}
    \| \hat{s}_{t'}(y_t) - \hat{s}_{t'}(y_0) \| &\le L_s \| y_t - y_0 \| \\
    &= L_s \left \|( e^{t} - 1)y_0 + 2(e^{ t}-1) \hat s_{t'}(y_0) + \sqrt{e^{2 t}-1} z \right\| \\
    &\le L_s ( e^{t}-1) \| y_0 \| + 2L_s (e^{t}-1) \| \hat s_{t'}(y_0) \| + L_s  \sqrt{e^{2 t}-1}\,\| z \| \\
    &\le 2L_s t \| y_0 \| + 4 L_s t \| \hat s_{t'}(y_0) \| + 2 L_s \sqrt{t} \| z \| \qquad \text{ since } e^{ t} - 1 \le 2  t.
\end{align*}
For the sake of subsequent analysis, we use a bound in terms of $\hat s_{t'}(y_t)$ rather than $\hat s_{t'}(y_0)$. Therefore, we use the following
\[\| \hat s_{t'}(y_0) \| - \| \hat s_{t'}(y_t) \| \le \| \hat s_{t'}(y_t) - \hat s_{t'}(y_0) \| \le L_s \| y_t - y_0 \|.\]
Then we can bound $\hat s_{t'}(y_0)$ as follows,
\begin{align*}
   \| \hat s_{t'}(y_0) \| &\le \| \hat s_{t'}(y_t) \| + L_s \| y_t - y_0 \| \\
   &\le \| \hat s_{t'}(y_t) \| + 2L_s t \| y_0 \| + 4 L_s t \| \hat s_{t'}(y_0) \| + 2 L_s \sqrt{t} \| z \| \\
   &\le \frac{1}{3}\| \hat s_{t'}(y_0) \| + \| \hat s_{t'}(y_t) \| + 2 L_s  t\| y_0 \| + 2 L_s  \sqrt{t} \| z \| \qquad \text{ since } t \le \frac{1}{12L_s}.
\end{align*}
Rearranging the above inequality, we have
\[ \| \hat s_{t'}(y_0) \| \le \frac{3}{2} \| \hat s_{t'}(y_t) \| +  3 L_s t \| y_0 \| + 3 L_s \sqrt{t} \| z \|.\]
Therefore,
\begin{align*}
    \| \hat{s}_{t'}(y_t) - \hat{s}_{t'}(y_0) \| &\le (2L_s t + 12L_s^2 t^2 ) \| y_0 \| + 6L_s t\| \hat s_{t'}(y_t) \| + (12L_s^2 t^{3/2} + 2L_s \sqrt{t})\| z \| \\
    &\le 3L_s t \| y_0 \| + 6L_s t\| \hat s_{t'}(y_t) \| + 3L_s \sqrt{t} \| z \| \qquad \text{since } t \le \frac{1}{12L_s}.
\end{align*}
Then we obtain the desired result,
\[ \| \hat{s}_{t'}(y_t) - \hat{s}_{t'}(y_0) \|^2 \le 36 t^2 L_s^2 \| y_0 \|^2 + 36 t L_s^2 \| z\|^2 + 72t^2L_s^2 \|\hat s_{t'}(y_t) \|.\]
\end{proof}

\subsection{Proof of Lemma \ref{lemma:deriv-kl-ddpm}}\label{sec:pf-lemma-deriv-kl-ddpm}
\begin{proof}[Proof of Lemma \ref{lemma:deriv-kl-ddpm}]
By Lemma \ref{lemma:continu-eq-bou-score}, we have
\[\frac{\partial \rho_{t}}{\partial t} = - \nabla \cdot (\rho_{t} y) - 2 \nabla \cdot \left( \rho_t \E_{\rho_{0 \mid t}}[\hat{s}_{T_0}(y_0) \mid Y_t = y]\right) + \Delta\rho_t.\]
On the other hand, the continuity equation of the true backward process~\eqref{eq:backwardOU} is
\[\frac{\partial \mu_t}{\partial t} = -\nabla \cdot (\mu_t y) - \Delta \mu_t.\]
It follows that
\begin{align*}
    \frac{\partial}{\partial t} H_{\mu_t}(\rho_t) &= \frac{d}{dt} \int \rho_t \log \frac{\rho_t}{\mu_t} dy \\
    &= \int \frac{\partial \rho_{t}}{\partial t} \log \frac{\rho_t}{\mu_t} dy + \int \mu_t \frac{\partial}{\partial t}\left( \frac{\rho_t}{\mu_t}\right) dy \\
    &= \int \frac{\partial \rho_{t}}{\partial t} \log \frac{\rho_t}{\mu_t} dy + \int \mu_t \left( \frac{1}{\mu_t} \frac{\partial \rho_t}{\partial t} - \frac{\rho_t}{\mu_t^2} \frac{\partial \mu_t}{\partial t}\right) dy \\
    &= \int \frac{\partial \rho_{t}}{\partial t} \log \frac{\rho_t}{\mu_t} dy + \int \frac{\partial \rho_t}{\partial t} dy - \int \frac{\rho_t}{\mu_t} \frac{\partial \mu_t}{\partial t} dy \\
    &= \int \frac{\partial \rho_{t}}{\partial t} \log \frac{\rho_t}{\mu_t} dy - \int \frac{\rho_t}{\mu_t} \frac{\partial \mu_t}{\partial t} dy \\
    &= \int \left[  - \nabla \cdot (\rho_{t} y) - 2 \nabla \cdot \left( \rho_t \E_{\rho_{0 \mid t}}[\hat{s}_{T_0}(y_0) \mid Y_t = y]\right) + \Delta\rho_t \right] \log \frac{\rho_t}{\mu_t} dy - \int \left[  -\nabla \cdot (\mu_t y) - \Delta \mu_t \right]  \frac{\rho_t}{\mu_t} dy \\
    &= \int \left[ \nabla \cdot (\mu_t y)\right] \frac{\rho_t}{\mu_t} dy - \int \left[ \nabla \cdot (\rho_t y)\right] \log \frac{\rho_t}{\mu_t} dy   + \int \Delta \rho_t \log \frac{\rho_t}{\mu_t} dy - \int \Delta \mu_t \frac{\rho_{t}}{\mu_{t}} dy \\
    & \qquad - 2\int \nabla \cdot \left( \rho_t \E_{\rho_{0 \mid t}}[\hat{s}_{T_0}(y_0) \mid Y_t = y]\right) \log \frac{\rho_{t}}{\mu_{t}} dy + 2 \int \Delta \mu_t \frac{\rho_{t}}{\mu_{t}} dy.
\end{align*}
By noting that
\begin{align*}
    &\int \left[ \nabla \cdot (\mu_t y)\right] \frac{\rho_t}{\mu_t} dy - \int \left[ \nabla \cdot (\rho_t y)\right] \log \frac{\rho_t}{\mu_t} dy  \\
    =& -\int \langle \mu_t y, \nabla \frac{\rho_t}{\mu_t} \rangle dy + \int \langle \rho_t y, \nabla \log \frac{\rho_t}{\mu_t} \rangle dy \qquad \text{by integration by parts} \\
    =& -\int \langle \rho_t y, \nabla \log \frac{\rho_t}{\mu_t}\rangle dy + \int \langle \rho_t y, \nabla \log \frac{\rho_t}{\mu_t} \rangle dy \\
    = & 0,
\end{align*}
\begin{align*}
    \int \Delta \rho_t \log \frac{\rho_t}{\mu_t} dy - \int \Delta \mu_t \frac{\rho_{t}}{\mu_{t}} dy &= -\int \langle \nabla \rho_t , \nabla \log \frac{\rho_t}{\mu_t} \rangle dy + \int \langle \nabla \mu_t , \nabla \frac{\rho_{t}}{\mu_{t}} \rangle dy \\
    &= -\int \langle \nabla \rho_t , \nabla \log \frac{\rho_t}{\mu_t} \rangle dy + \int \langle \nabla \mu_t , \frac{\rho_t}{\mu_t} \nabla \log \frac{\rho_t}{\mu_t}\rangle dy \\
    &= -\int \rho_t \langle \nabla \log \rho_t , \nabla \log \frac{\rho_t}{\mu_t} \rangle dy + \int \rho_t \langle \nabla \log \mu_t ,  \nabla \log \frac{\rho_t}{\mu_t}\rangle dy \\
    &= -\int \rho_t \| \nabla \log \frac{\rho_t}{\mu_t} \|^2 dy\\
    &= -J_{\mu_{t}}(\rho_t),
\end{align*}
and
\[ 2 \int \Delta \mu_t \frac{\rho_{t}}{\mu_{t}} dy = -2\int \langle \nabla \mu_t, \nabla \frac{\rho_{t}}{\mu_{t}} \rangle dy = -2 \int \rho_t \langle \nabla \log \mu_t, \log \frac{\rho_t}{\mu_t} \rangle dy,\]
we obtain
\begin{align*}
    \frac{\partial}{\partial t}H_{\mu_{t}}(\rho_{t}) &= -J_{\mu_{t}}(\rho_t) + 2\int \rho_t \langle  \E_{\rho_{0 \mid t}}[\hat{s}_{T_0}(y_0) \mid Y_t = y] -\nabla \log \mu_t, \nabla \log \frac{\rho_{t}}{\mu_{t}}\rangle dy \\
    &= -J_{\mu_{t}}(\rho_t) + 2\E_{\rho_{0t}}\left[ \langle  \hat{s}_{T_0}(y_0) -\nabla \log \mu_t(y_t), \nabla \log \frac{\rho_{t}(y_t)}{\mu_{t}(y_t)}\rangle \right] \qquad \text{by renaming $y$ as $y_t$} \\
    &\le -J_{\mu_{t}}(\rho_t) + 4 \E_{\rho_{0t}}\left[ \| \hat{s}_{T_0}(y_0) -\nabla \log \mu_t(y_t) \|^2 \right] + \frac{1}{4}\E_{\rho_{t}}\left[ \| \nabla \log \frac{\rho_{t}}{\mu_{t}}\|^2 \right] \\
    &= -\frac{3}{4}J_{\mu_{t}}(\rho_t) + 4 \E_{\rho_{0t}}\left[ \| \hat{s}_{T_0}(y_0) -\nabla \log \mu_t(y_t) \|^2 \right] \\
    &= -\frac{3}{4}J_{\mu_{t}}(\rho_t) + 4 \E_{\rho_{0t}}\left[ \| \hat{s}_{T_0}(y_0) - \hat{s}_{T_0}(y_t) + \hat{s}_{T_0}(y_t) -\nabla \log \mu_t(y_t) \|^2 \right] \\
    &\le -\frac{3}{4}J_{\mu_{t}}(\rho_t) + 8 \E_{\rho_{0t}}\left[ \| \hat{s}_{T_0}(y_0) - \hat{s}_{T_0}(y_t)\|^2 \right] + 8 \E_{\rho_{t}}\left[ \| \hat{s}_{T_0}(y_t) -\nabla \log \mu_t(y_t) \|^2 \right]
\end{align*}
By Lemma \ref{lemma:score-lipschitz}, we can bound the second term above as follows
\begin{align*}
    \E_{\rho_{0t}}\left[ \| \hat{s}_{T_0}(y_0) - \hat{s}_{T_0}(y_t)\|^2 \right] &\le  \E_{\rho_{0t}}\left[36 t^2 L_s^2 \| y_0 \|^2 + 36 L_s^2 t \| z \|^2 + 72 L_s^2 t^2 \| \hat{s}_{T_0}(y_t) \|^2 \right] \\
    &= 36 t^2 L_s^2 d +  36 L_s^2 td +  72 L_s^2 t^2 \E_{\rho_{t}}\left[\| \hat{s}_{T_0}(y_t) \|^2 \right] \\
    &\le 36
    t^2 L_s^2 d +  36 L_s^2 td +  144 L_s^2 t^2 \big( \E_{\rho_{t}}\left[\| \hat{s}_{T_0}(y) - \nabla \log \mu_t(y) \|^2 \right] \\
    &\qquad + \E_{\rho_{t}}\left[\| \nabla \log \mu_t(y) \|^2 \right] \big).
\end{align*}
Since $\mu_t = \nu_{T_0-t}$ is $L$-smooth, by \cite[Lemma 16]{CEL+22} we have
\[\E_{\rho_{t}}\left[\| \nabla \log \mu_t(y) \|^2 \right] \le J_{\mu_{t}}(\rho_{t}) + 2dL.\]
Then we apply the Donsker-Varadhan variational characterizations of KL divergence $\E_{P}[f(x)] \leq \log \E_{Q}e^{f(x)} + H_Q(P)$ to perform a change of measure,
\begin{align*}
  \E_{\rho_{t}}\left[\| \hat{s}_{T_0} - \nabla \log \mu_t \|^2 \right] &\le \frac{6 \alpha_{T_0 - t}}{65} \log \E_{\mu_{t}}[\exp(\frac{65}{6\alpha_{T_0 - t}}\|\hat s_{t} - \nabla \log \mu_t \|^2)] + \frac{6 \alpha_{T_0 - t}}{65} H_{\mu_{t}}(\rho_t) \\
  &\le \error_{\mgf}^2 + \frac{6 \alpha_{T_0 - t}}{65} H_{\mu_{t}}(\rho_t).
\end{align*}
Putting everything together and using $t^2 \le \step^2 \le 1/96^2 L_s^2 L^2$, we obtain the desired result
\[ \frac{\partial}{\partial t} H_{\mu_{t}}(\rho_t) \le  -\frac{\alpha_{T_0 - t}}{2} H_{\mu_{t}}(\rho_t) + \frac{65}{8} \error_{\mgf}^2 + 9 L_s (3 + 32 L_s)\, dt. \]
\end{proof}

\subsection{Proof of Lemma \ref{lemma:one-step-ddpm}}
\label{sec:pf-lemma-one-step-ddpm}
\begin{proof}[Proof of Lemma \ref{lemma:one-step-ddpm}]
Let $A_t \coloneqq \int_0^t \alpha_{T_0 - s} ds = \int_0^t \frac{\alpha}{\alpha + (1-\alpha) e^{-2(T_0 - s)}} ds$. Then $\dot A_t = \alpha_{T_0 - t}$.
Following Lemma \ref{lemma:deriv-kl-ddpm}, we have
\[\frac{\partial}{\partial t}e^{\frac{A_t}{2}} H_{\mu_{t}}(\rho_{t}) \le e^{\frac{A_t}{2}} \left( \frac{65}{8} \error_{\mgf}^2 + 9 L_s (3 + 32 L_s)\, dt \right).\]
Integrating from $0$ to $\step$ yields
\begin{align*}
    H_{\mu_{h}}(\rho_{h}) \le e^{-\frac{A_\step}{2}}  H_{\mu_{0}}(\rho_{0}) + \int_0^\step e^{\frac{A_t - A_\step}{2}} \left( \frac{65}{8} \error_{\mgf}^2 + 9 L_s (3 + 32 L_s)\, dt \right) \textnormal{d} t 
\end{align*}
Note that $A_t =\frac{1}{2} \log (\alpha e^{2T_0} + 1 -\alpha) - \frac{1}{2} \log(\alpha e^{2(T_0-t)} + 1 -\alpha)$, thus $A_t - A_\step \le 0$ for $t \le \step$. Therefore,
\begin{align*}
     H_{\mu_{h}}(\rho_{h}) &\le e^{-\frac{A_\step}{2}}  H_{\mu_{0}}(\rho_{0}) + \int_0^\step \left( \frac{65}{8} \error_{\mgf}^2 + 9 L_s (3 + 32 L_s)\, dt \right)\textnormal{d} t \\
     &= \left(\frac{\alpha e^{2T_1} + 1-\alpha}{\alpha e^{2T_0} + 1-\alpha}\right)^{1/4}  H_{\mu_{0}}(\rho_{0}) + \frac{65}{8} \error_{\mgf}^2 h+ \frac{9}{2} L_s (3 + 32 L_s)\, d\step^2
\end{align*}
Renaming $\rho_0 = \rho_k$, $\rho_h = \rho_{k+1}$, $\mu_0 = \mu_k$ and $\mu_h = \mu_{k+1}$, we get the desired contraction
\[H_{\mu_{k+1}}(\rho_{k+1}) \le \left(\frac{\alpha e^{2T_{k+1}} + 1-\alpha}{\alpha e^{2T_k} + 1-\alpha}\right)^{1/4}  H_{\mu_{k}}(\rho_{k}) + \frac{65}{8} \error_{\mgf}^2 h+ \frac{9}{2} L_s (3 + 32 L_s)\, d\step^2.\]
\end{proof}

\subsection{Proof of Lemma \ref{lem:conv-ou}}
\label{sec:pf-lsi-ou}
Since we are measuring KL divergence with $\nu_T$ (instead of $\gamma$), we need a control on how the LSI constant evolves along the OU process, as stated in the following lemma. 
\begin{lemma}\label{lemma:lsi-ou}
Let $X_0 \sim \nu_0 = \nu$ where $\nu$ is $\alpha$-LSI ($\alpha > 0$) evolve along the following OU process targeting $\N(0, \beta^{-1}I_d)$:
\begin{align}\label{eq:ou-sde}
dX_t = -\beta X_t dt + \sqrt{2} dW_t.
\end{align}
At time $t$, $X_t \sim \nu_t$ where $\nu_t$ is $\alpha_t$-LSI and $\alpha_t = \frac{\alpha \beta}{\alpha + (\beta - \alpha) e^{-2\beta t}}$. In particular, if $\beta = \alpha$, then $\alpha_t = \alpha$. 
\end{lemma}
\begin{proof}[Proof of Lemma \ref{lemma:lsi-ou}]
Eq.~\eqref{eq:ou-sde} is equivalent to $ d(e^{\beta t} X_t) = \sqrt{2} e^{\beta t} dt$, therefore we have
\[X_t \overset{d}{=} e^{-\beta t} X_0 + \sqrt{\frac{1-e^{-2\beta t}}{\beta}} Z,\]
where $Z$ is a standard Gaussian in $\R^d$. 
By \cite[Lemma 16, 17]{VW19}, the distribution of $e^{-\beta t} X_0$ satisfies LSI with constant $\alpha e^{2\beta t}$ and the LSI constant of $\nu_t$ is \[\left( \alpha^{-1} e^{-2 \beta t} + \frac{1-e^{-2\beta t}}{\beta}\right)^{-1} = \frac{\alpha \beta }{\alpha + (\beta - \alpha) e^{-2\beta t}} \in (\min(\alpha, \beta), \max(\alpha, \beta)).\]
In particular, when $\beta = \alpha$, we have $\alpha_t \equiv \alpha$. 
\end{proof}

\begin{proof}[Proof of Lemma \ref{lem:conv-ou}]
\begin{align*}
    \frac{d}{dt} H_{\nu_t} (\gamma) &= -\int \frac{\gamma}{\nu_t} \frac{\partial \nu_t}{\partial t} dx \\
    &= -\int \frac{\gamma}{\nu_t} \;\nabla \cdot (\nu_t \nabla \log \frac{\nu_t}{\gamma}) dx \\
    &= \int \langle \nabla \frac{\gamma}{\nu_t}, \;\nu_t \nabla \log \frac{\nu_t}{\gamma} \rangle \qquad \text{by integration by parts} \\
    &=  \int \langle \frac{\gamma}{\nu_t} \nabla \log \frac{\gamma}{\nu_t}, \;\nu_t \nabla \log \frac{\nu_t}{\gamma} \rangle \\
    &=  \int \gamma \langle \nabla \log \frac{\gamma}{\nu_t}, \; \nabla \log \frac{\nu_t}{\gamma} \rangle \\
    &= -\mathbb{E}_{\gamma}[\| \nabla \log \frac{\gamma}{\nu_t} \|^2] \\
    &= - J_{\nu_t}(\gamma) \\
    &\le -2 \alpha_t H_{\nu_t} (\gamma) \qquad \text{since $\nu_t$ is $\alpha_t$-LSI by Lemma \ref{lemma:lsi-ou}}. 
\end{align*}
This is equivalent to 
\begin{align*}
&\frac{d}{dt} \log H_{\nu_t} (\gamma) \le -2\alpha_t \\
\implies & \log  H_{\nu_T} (\gamma) - \log  H_{\nu_0} (\gamma) \le -2 \int_0^T \frac{\alpha}{\alpha + (1-\alpha) e^{-2t}} \textnormal{d}t = -\log (\alpha e^{2T} + 1-\alpha)
\end{align*}
Therefore, we obtain
\[H_{\nu_T} (\gamma) \le \frac{H_{\nu_0}(\gamma)}{\alpha e^{2T} + 1-\alpha}.\]
\end{proof}

\subsection{Proof of Theorem \ref{thm:conv-ddpm}}
\label{sec:pf-conv-ddpm2}
\begin{proof}[Proof of Theorem \ref{thm:conv-ddpm}]
Let $B_{k} = (\alpha e^{2T_k} + 1 -\alpha)^{1/4}$ for $k=0, 1, \cdots, K$.
Applying the recursion in Lemma \ref{lemma:one-step-ddpm} $K$ times, we obtain
\begin{align}
\label{eq:convergence-backward-diffusion}
  H_{\mu_{K}}(\rho_{K}) &\le \frac{B_K}{B_0}  H_{\mu_{0}}(\rho_{0}) +\sum_{i=0}^{K}\frac{B_K}{B_i} \left(  \frac{65}{8} \error_{\mgf}^2 h+ \frac{9}{2} L_s (3 + 32 L_s)\, d\step^2\right).
\end{align}
Recall that $B_K = 1$ and $T_i = T-ih = (K-i)h$,
\begin{align*}
    \sum_{i=0}^{K}\frac{B_K}{B_i} = \sum_{i=0}^{K} \frac{1}{(\alpha e^{2T_i} + 1 -\alpha)^{1/4}}.
\end{align*}
By noting that $(\alpha e^{2ih} + 1 -\alpha)^{-1/4} \le \min(\alpha^{-1/4} e^{-ih/2}, (1-\alpha)^{-1/4})$ and $\alpha^{-1/4} e^{-ih/2} \le (1-\alpha)^{-1/4} \iff i \ge \frac{1}{2h} \log \frac{1-\alpha}{\alpha}$, and letting $I =\frac{1}{2h} \log \frac{1-\alpha}{\alpha} $, the summation can be bounded as follows:
\begin{align*}
    \sum_{i=0}^K \frac{1}{B_i} &\le \sum_{i=0}^{I-1} (1-\alpha)^{-1/4} + \sum_{i=I}^K \alpha^{-1/4} e^{-ih/2} \\
    &\le I (1-\alpha)^{-1/4} + \alpha^{-1/4} \frac{e^{-Jh/2}}{1-e^{-h/2}} \\
    &\le \frac{1}{2h (1-\alpha)^{1/4}} \log \frac{1}{\alpha} + \frac{8}{3h (1-\alpha)^{1/4}}.
\end{align*}
Since $\alpha < 1/2$, then
$$\sum_{i=0}^K \frac{1}{B_i} \lesssim \frac{1}{h} \log \frac{1}{\alpha}.$$

Since $1-e^{-c} \ge \frac{3}{4}c$ for $0 < c= \frac{\step }{2} \le \frac{1}{4}$ which is satisfied, we have
\begin{align}
\label{eq:bound-kl-ddpm}
    H_{\mu_{K}}(\rho_{K}) &\lesssim \alpha^{-1/4} e^{-\frac{Kh}{2}} H_{\mu_{0}}(\rho_{0}) + \left( \error_{\mgf}^2 + L_s^2\, d\step\right)\log \frac{1}{\alpha}.
\end{align}
By Lemma~\ref{lem:conv-ou}, the following holds
\begin{align}
\label{eq:convergence-forward-process}
H_{\mu_{0}}(\rho_{0}) = H_{\nu_T} (\gamma) \le \frac{H_{\nu_0}(\gamma)}{\alpha e^{2T} + 1-\alpha} \le \frac{e^{-2Kh}}{\alpha} H_{\nu}(\gamma).   
\end{align}
Plugging it into~\eqref{eq:bound-kl-ddpm} gives us the desired bound
\[H_{\nu}(\rho_{K}) \lesssim \alpha^{-5/4} e^{-\frac{5Kh}{2}} H_{\nu}(\gamma)+ \left(\error_{\mgf}^2 + L_s^2\, d\step\right)\log \frac{1}{\alpha}.\]

\end{proof}

\section{Proof of Lemma~\ref{lem-subgau-kde}}

We use the following alternative definition of sub-Gaussian to prove Lemma \ref{lem-subgau-kde}. 
\begin{lemma}[Theorem 2.6 in \cite{Wainwright2019}]
If $\rho$ is $\sigma$-sub-Gaussian, then for all $0 \le \lambda < 1$:
\[\mathbb{E}_{\rho}\left[\exp\left(\frac{\lambda \|X\|^2}{2\sigma^2}\right)\right] \le \frac{1}{(1-\lambda)^{1/2}}.\]
\end{lemma}

\label{sec:pf-subgau-kde}
\begin{proof}[Proof of Lemma \ref{lem-subgau-kde}]
By \cite[Lemma 13]{CCL+22},
\begin{align*}
    \|s_\eta(x) - s(x)\| &\lesssim L \sqrt{\eta d} + L\eta \|s(x)\| \\
    &\le L \sqrt{\eta d} + L\eta (\|s(0)\| + L \|x\|)
\end{align*}
where $\lesssim$ hides absolute constants. If $\eta \le d / \|s(0)\|^2$, then
\[\|s_\eta(x) - s(x)\| \lesssim 2L \sqrt{\eta d} + L^2 \eta \|x\|\]
It follows that $\|s_\eta(x) - s(x)\|^2 \lesssim 8L^2 \eta d + 2L^4 \eta^2 \|x\|^2.$ Then we have
\begin{align*}
    \E_\rho[\exp[r\|s_\eta-s\|^2]]
    &\lesssim \exp(8 r L^2 \eta d) \, \E_\rho[\exp(2 r L^4 \eta^2 \|x\|^2)]
    \le \frac{\exp(8 r L^2 \eta d)}{(1-4\sigma^2 r L^4 \eta^2)^{1/2}}
\end{align*}
as long as $4\sigma^2 r L^4 \eta^2 < 1$, i.e.\
$\eta < \frac{1}{2 \sigma \sqrt{r} L^2}$. By noting that $1-x \ge e^{-2x}$ for $0 \le x \le \frac{1}{2}$, we have if
$\eta \le \frac{1}{2 \sqrt{2} \sigma \sqrt{r} L^2}$,
\begin{align*}
    \E_\rho[\exp(r\|s_\eta - s\|^2)]
    &\lesssim \exp(8 r L^2 \eta d + 4\sigma^2 r L^4 \eta^2).
\end{align*}
Therefore, we obtain the desired error bound
\begin{align*}
    \frac{1}{r}\log \E_\rho[e^{r\|s_\eta - s\|^2}]
    \lesssim \eta L^2 (d + \sigma^2\eta L^2).
\end{align*}
\end{proof}

\section{Smoothness is preserved along the heat flow}
\label{sec:pf-heat-flow-smooth}

\begin{lemma}\label{lemma:heat-flow-smooth}
Assume $\rho \propto e^{-f}$ where $f$ is $L$-smooth. Let it evolve along the heat flow,
then at time $t \in (0, \frac{1}{2L})$,
$ \rho_t =  \rho * \N(0, tI_d)$ is $2L$-smooth.
\end{lemma}

\begin{proof}
First, we derive 
\begin{align}
\label{eq:score}
    s_t(y) = \nabla \log \rho_t (y) =\frac{\E_{\rho_{0 \mid t=y}}[X] - y}{t}.
\end{align}
Note that $\rho_{t \mid 0}(y \mid x) = (2\pi t) ^{-d/2}\exp(-\frac{\| y - x \|^2}{2t})$ and 
$\nabla\, \rho_{t \mid 0}(y \mid x) = \rho_{t \mid 0}(y \mid x) \frac{(x - y)}{t}$
where $\nabla = \nabla_y$ is derivative with respect to $y$.
Since
$\rho_t = \rho \ast \N(0, tI_d) = \int \rho(x) \rho_{t \mid 0}(y \mid x) dx$, the 
score function at time $t$ is
\begin{align*}
    s_t(y) &= \nabla \log \rho_t(y) \\
    &= \frac{\int \rho(x) \nabla \rho_{t \mid 0}(y \mid x) dx}{\rho_t(y)} \\
    &= \frac{\int \rho(x) \rho_{t \mid 0}(y \mid x)  \frac{(x - y)}{t} dx}{\rho_t(y)} \\
    &= \int \rho_{0 \mid t=y}(x \mid y) \left(\frac{x - y}{t} \right) dx \\
    &=\frac{\E_{\rho_{0 \mid t=y}}[X] - y}{t}. 
\end{align*}
Next, we derive the following
\begin{align}\label{Eq:HessRhot}
    -\nabla^2 \log \rho_t (y) = \frac{I_d}{t}- \frac{\text{Cov}_{\rho_{0 \mid t=y}}[X]}{t^2}.
\end{align}
Noting that 
\begin{align*}
    \nabla \rho_{0 \mid t} (y \mid x) &= \nabla\, \frac{\rho_{t \mid 0} (y\mid x) \rho(x)}{\rho_t(y)} \\
    &=\frac{\nabla \,\rho_{t \mid 0} (y\mid x) \rho(x)}{\rho_t (y)} - \frac{\rho_{t \mid 0} (y\mid x) \rho(x) \nabla \rho_t(y)}{ \rho_t^2(y)} \\
    &=  \rho_{0 \mid t} (y \mid x) \frac{x-y}{t} - \rho_{0 \mid t} (y \mid x) \nabla \log \rho_t(y).
\end{align*}
we have the gradient of posterior mean is
\begin{align*}
    \nabla \E_{\rho_{0 \mid t=y}}[X] &= \int \nabla \rho_{0 \mid t} (y \mid x) x^T dx \\
    &= \E_{\rho_{0 \mid t=y}}\left[\frac{ (x-y)x^T}{t} - \nabla \log \rho_t(y) x^T \right] \\
    &= \frac{\E_{\rho_{0 \mid t=y}}[XX^T]}{t} - \frac{\E_{\rho_{0 \mid t=y}}[X] \E_{\rho_{0 \mid t=y}}[X^T]}{t} \qquad \text{by Eq~\eqref{eq:score}} \\
    &= \frac{\text{Cov}_{\rho_{0 \mid t=y}}[X]}{t}.
\end{align*}
Hence, we obtain
\begin{align*}
    -\nabla^2 \log \rho_t (y) = \frac{I_d}{t}- \frac{\text{Cov}_{\rho_{0 \mid t=y}}[X]}{t^2}.
\end{align*}

We now bound the covariance term for any $y \in \R^d$.
Since $\rho_{0 \mid t}(x \mid y) \propto e^{-f(x) - \frac{1}{2t} \| y-x \|^2}$, we have
\[-\nabla_x^2 \log \rho_{0 \mid t}(x\mid y) = \nabla_x^2 \left(f(x) + \frac{1}{2t} \| y-x \|^2\right) = \nabla^2f(x) + \frac{1}{t} I_d\]
(note the derivative above is with respect to $x$).
Since $\nabla^2f(x) \preceq LI_d$, we have 
\[-\nabla_x^2 \log \rho_{0 \mid t}(x\mid y)\preceq (L + \frac{1}{t}) I_d.\]
This implies (see Lemma~\ref{Lem:Cov} below):
\[\Cov_{\rho_{0 \mid t=y}}[X] \succeq \frac{1}{L + 1/t} I_d.\] Therefore, we obtain an upper bound of the Hessian matrix~\eqref{Eq:HessRhot}:
\begin{align}
\label{eq:upbd}
-\nabla^2 \log \rho_t (y) \preceq \left( \frac{1}{t} - \frac{1}{t (tL+1)} \right) I_d = \frac{L}{tL + 1} I_d.
\end{align}

To get a lower bound, we note that since $\nabla^2 f(x) \succeq -LI_d$,
\[-\nabla_x^2 \log \rho_{0 \mid t}(x\mid y) \succeq (-L + \frac{1}{t}) I_d \succeq 0\]
so for $t < \frac{1}{L}$, $\rho_{0 \mid t}(\cdot \mid y)$ is $(\frac{1}{t}-L)$-strongly log-concave, which implies
\[\Cov_{\rho_{0 \mid t=y}}[X] \preceq \frac{1}{1/t-L} I_d.\]
Therefore,
\begin{align}
\label{eq:lowbd}
-\nabla^2 \log \rho_t (y) \succeq \left(\frac{1}{t}- \frac{1}{t(1-tL)}\right)I_d = - \frac{L}{1-tL}I_d.
\end{align}
Combining Eq.~\eqref{eq:upbd} and \eqref{eq:lowbd}:
\[ - \frac{L}{1-tL}I_d \preceq -\nabla^2 \log \rho_t (y) \preceq \frac{L}{1+tL} I_d.\]
For $0 \le t < \frac{1}{L}$, $\frac{L}{1-tL} \ge \frac{L}{1+tL}$.
Therefore, $\rho_t$ is $\frac{L}{1-tL}$-smooth. If $t \le \frac{1}{2L}$, then we have $\frac{L}{1-tL} \le 2L$, so we conclude $\rho_t$ is $2L$-smooth for $0 \le t \le \frac{1}{2L}$.
\end{proof}

We also have the following estimate which appears in   \cite{KP2021, BH1976}.
Here we provide an alternate proof.

\begin{lemma}\label{Lem:Cov}
Suppose a density $\rho$ satisfies $-\nabla^2 \log \rho(x) \preceq L I.$
    Then
    \[\Cov_\rho(X) \succeq \frac{1}{L} I.\]
\end{lemma}
\begin{proof}
    Let $\nu = \N(m,C)$ be a Gaussian with the same mean $m = \E_\rho[X]$ and covariance $C = \Cov_\rho(X)$ as $\rho$.
    Note $-\nabla \log \nu(x) = C^{-1}(x-m)$.
    By calculation, we can show that the relative Fisher information matrix is:
    \begin{align*}
        \tilde J_\nu(\rho) &\coloneqq \E_\rho\left[\left(\nabla \log \frac{\rho}{\nu}\right)\left(\nabla \log \frac{\rho}{\nu}\right)^\top \right] \\
        &= \E_\rho\left[\left(\nabla \log \rho\right)\left(\nabla \log \rho\right)^\top \right] + 
        \E_\rho\left[\left(\nabla \log \rho\right)\left(C^{-1}(x-m)\right)^\top \right]  \\
        &\qquad +  \E_\rho\left[\left(C^{-1}(x-m)\right)\left(\nabla \log \rho\right)^\top \right] + \E_\rho\left[\left(C^{-1}(x-m)\right)\left(C^{-1}(x-m)\right)^\top \right] \\
        &= \E_\rho[-\nabla^2 \log \rho] - C^{-1} - C^{-1} + C^{-1} C C^{-1} \\
        &\preceq LI - C^{-1}
    \end{align*}
    where the third equality above holds by integration by parts, and the last inequality holds by $L$-smoothness of $\rho$.
    Since $\tilde J_\nu(\rho) \succeq 0$, this implies $C^{-1} \preceq LI$ or equivalently 
    $C \succeq \frac{1}{L} I$, as desired.
\end{proof}

\end{document}